\def\year{2022}\relax
\documentclass[letterpaper]{article} 
\usepackage{aaai22}  
\usepackage{times}  
\usepackage{helvet}  
\usepackage{courier}  
\usepackage[hyphens]{url}  
\usepackage{graphicx} 
\usepackage{natbib}  
\usepackage{caption} 
\DeclareCaptionStyle{ruled}{labelfont=normalfont,labelsep=colon,strut=off} 
\usepackage{algorithm}
\usepackage{algorithmic}

\usepackage{hyperref}       
\usepackage{booktabs}       
\usepackage{amsfonts}       
\usepackage{nicefrac}       
\usepackage{microtype}      
\usepackage{xcolor}         

\usepackage{subfigure}
\usepackage{amsmath}
\usepackage{amsthm}
\usepackage{mathtools}
\usepackage{wrapfig}
\usepackage{enumerate}   

\newtheorem{theorem}{Theorem}
\newtheorem{corollary}{Corollary}[theorem]
\newtheorem{lemma}[theorem]{Lemma}
\newtheorem{proposition}[theorem]{Proposition}
\theoremstyle{definition}

\newtheorem{assumption}{Assumption}

\usepackage{newfloat}
\usepackage{listings}
\floatstyle{ruled}
\newfloat{listing}{tb}{lst}{}
\floatname{listing}{Listing}
\title{A Fully Single Loop Algorithm for Bilevel Optimization without Hessian Inverse}
\author {
    Junyi Li,\textsuperscript{\rm 1}
    Bin Gu,\textsuperscript{\rm 2}
    Heng Huang \textsuperscript{\rm 1}\thanks{This work was supported by NSF IIS 1845666, 1852606, 1838627, 1837956, 1956002, IIA 2040588.}
}
\affiliations {
    \textsuperscript{\rm 1}  Electrical and Computer Engineering, University of Pittsburgh, PA, USA
    \textsuperscript{\rm 2}  MBZUAI, United Arab Emirates\\
    junyili.ai@gmail.com,  jsgubin@gmail.com, heng.huang@pitt.edu
}

\usepackage{bibentry}

\begin{document}

\maketitle

\begin{abstract}
	In this paper, we propose a new Hessian inverse free Fully Single Loop Algorithm (FSLA) for bilevel optimization problems. Classic algorithms for bilevel optimization admit a double loop structure which is computationally expensive. Recently, several single loop algorithms have been proposed with optimizing the inner and outer variable alternatively. However, these algorithms not yet achieve fully single loop. As they overlook the loop needed to evaluate the hyper-gradient for a given inner and outer state. In order to develop a fully single loop algorithm, we first study the structure of the hyper-gradient and identify a general approximation formulation of hyper-gradient computation that encompasses several previous common approaches, \emph{e.g.} back-propagation through time, conjugate gradient, \emph{etc.} Based on this formulation, we introduce a new state variable to maintain the historical hyper-gradient information. Combining our new formulation with the alternative update of the inner and outer variables, we propose an efficient fully single loop algorithm. We theoretically show that the error generated by the new state can be bounded and our algorithm converges with the rate of $O(\epsilon^{-2})$. Finally, we verify the efficacy our algorithm empirically through multiple bilevel optimization based machine learning tasks.
\end{abstract}

\section{Introduction}
In this paper, we study the bilevel optimization problem, which includes two levels of optimization: an outer problem and an inner problem. The outer problem depends on the solution of the inner problem. Many machine learning tasks can be formulated as a bilevel optimization problem, such as hyper-parameter optimization~\cite{lorraine2018stochastic}, meta learning~\cite{franceschi2018bilevel}, Stackelberg game model~\cite{ghadimi2018approximation}, equilibrium model~\cite{grazzi2020iteration}, \emph{etc.} However, Bilevel optimization is challenging to solve. The  gradient-based algorithm~\cite{ghadimi2018approximation} requires a double loop structure. For each inner loop, the inner problem is solved with the given outer state. In the outer loop, the hyper-gradient (the gradient \emph{w.r.t} the outer variable) is evaluated based on the solution of the inner loop and the outer state is updated with a gradient-based optimizer (such as SGD). This simple double loop algorithm is guaranteed to converge under mild assumptions and works well for small scale problems. But when the inner problem is in large scale, the double loop algorithm is very slow and becomes impractical.

In fact, hyper-gradient evaluation is the major bottleneck of bilevel optimization. The solution of the inner problem is usually implicitly defined over the outer state, and naturally its gradient \emph{w.r.t} the outer variable is also implicitly defined. To evaluate hyper-gradient, we need to approximate this implicit gradient via an iterative algorithm (the inner loop). Various algorithms have been proposed for hyper-gradient evaluation~\cite{ferris1991finite, ghadimi2018approximation, grazzi2020iteration,lorraine2018stochastic, liao2018reviving}. These methods were designed from various perspectives and based on different techniques, thus they look quite different on the first sight. However, we can use a general formulation to incorporate all these methods under one framework. Roughly, the hyper-gradient evaluation can be expressed as a finite sum of terms defined over a sequence of inner states and momentum coefficients, and these states and coefficients are chosen differently in distinct algorithms. Our general formulation provides a new perspective to help understand the properties of these classic methods. In particular, we derive a sufficient condition such that the general formulation converges to the exact hyper-gradient.

One benefit of our general formulation is to inspire new algorithms for bilevel optimization. Based on our general formulation, we propose a new fully single loop algorithm named as `FSLA'. Compared to previous single loop algorithms~\cite{guo2021stochastic,chen2021single, ji2020provably, khanduri2021near, huang2021biadam}, our FSLA does not require any inner loop. In literature~\cite{ji2020provably, khanduri2021near}, the existing methods either reuse the last hyper-iteration's inner solution as a warm start of the current iteration, or just alternatively update the inner and outer variables with carefully designed learning rate schedule~\cite{hong2020two}. They also utilize the variance-reduction techniques to control the variance~\cite{khanduri2021near}. However, these methods focus on solving the inner problem and pay little attention to the hyper-gradient evaluation process which also requires a loop. In our new algorithm, we introduce a new state $v_k$ to keep track of the historical hyper-gradient information. During each iteration, we perform one step update. We study the bias caused by $v_k$ and theoretically show that the convergence of our new algorithm is with rate $O(\epsilon^{-2})$. The main contributions of this paper can be summarized as follows:
\begin{enumerate}
	\item We propose a general formulation to unify different existing  hyper-gradient approximation methods under the same framework, and identify the sufficient condition on which it converges to the exact hyper-gradient;
	\item We propose a new fully single loop algorithm for bilevel optimization. The new algorithm avoids the need of time-consuming Hessian inverse by introducing a new state to track the historical hyper-gradient information;
	\item We prove that our new algorithm has fast O($\epsilon^{-2}$) convergence rate for the nonconvex-strongly-convex case, and we validate effectiveness of our algorithm over different bilevel optimization based machine learning tasks.
\end{enumerate}

\noindent\textbf{Organization:} The rest of this paper is
organized as follows: in Section 2, we briefly review the recent development of Bilevel optimization; in Section 3, we introduce the general formulation of hyper-gradient approximation and the sufficient condition of convergence; in Section 4, we formally propose our new fully single loop algorithm, FSLA; in Section 5, we present the convergence result of our algorithm; in Section 6, we perform experiments to validate our proposed methods; in Section 7, we conclude and summarize the paper.

\noindent\textbf{Notations:} We use $\nabla_x$ to denote the full gradient \emph{w.r.t.} the variable $x$, where the subscript is omitted if clear from the context, and $\partial_{x}$ denotes the partial derivative. Higher order derivatives follow similar rules. $||\cdot||$ represents $\ell_2$-norm for vectors and spectral norm for matrices. $[K]$ represents the sequence from 0 to $K$. $\Pi_{i=m}^{n}A_i = A_m\times\dots A_n$ if $m \le n$, and $\Pi_{i=m}^{n}A_i = I$ if $m > n$.

\section{Related Works}
Bilevel optimization dates back to the 1960s when ~\citet{willoughby1979solutions} proposed a regularization method, and then followed by many research works~\cite{ferris1991finite,solodov2007explicit,yamada2011minimizing,sabach2017first}. In machine learning community, similar ideas in the name of implicit differentiation were also used in Hyper-parameter Optimization for a long time~\cite{larsen1996design,chen1999optimal,bengio2000gradient, do2007efficient}. However, implicit differentiation needs to compute an accurate inner problem solution in each update of the outer variable, which leads to high computational cost for large-scale problems. Thus, researchers turned to solve the inner problem with a fix number of steps, and computed the gradient \emph{w.r.t} the outer variables with the `back-propagation through time' technique~\cite{domke2012generic, maclaurin2015gradient, franceschi2017forward, pedregosa2016hyperparameter, shaban2018truncated}. \citet{domke2012generic} considered the case when the inner optimizer is gradient-descent, heavy ball, and LBFGS method, and derived their reversing dynamics with the energy models as example; \citet{maclaurin2015gradient} considered momentum-based SGD method, furthermore, \citet{franceschi2017forward} discussed two modes, a forward mode and a backward mode, and compared the trade-off of these two modes in terms of memory-usage and computation efficiency; \citet{pedregosa2016hyperparameter} studied the influence of inner solution errors to the convergence of bilevel optimization; \citet{shaban2018truncated} proposed to truncate the back propagation path to save computation. 

The back-propagation based methods work well in practice, but the number of inner steps usually relies on trial and error, furthermore, it is still expensive to perform multiple inner steps for modern machine learning models with hundreds of millions of parameters. Recently, it witnessed a surge of interest in using implicit differentiation to derive single loop algorithms. \citet{ghadimi2018approximation} introduced BSA, an accelerated approximate implicit differentiation method with Neumann Series. \citet{hong2020two} proposed TTSA, a single loop algorithm with a two-timescale learning rate schedule.
\citet{ji2020provably} and~\citet{ji2021lower} presented warm start strategy to reduce the number of inner steps needed at each iteration. \citet{khanduri2021near} designed SUSTAIN which applied variance reduction technique~\cite{cutkosky2019momentum, huang2021super} over both the inner and the outer variable. \citet{chen2021single} proposed STABLE, a single loop algorithm accumulating the Hessian matrix, and achieved the same order of sample complexity as the single-level optimization problems. \citet{yang2021provably} proposed two algorithms: MRBO and VRBO. The MRBO method uses double variance reduction trick and resembles SUSTAIN, while VRBO is based on SARAH/SPIDER. \citet{huang2021enhanced} proposed BiO-BreD which is also based on the variance reduction technique with a better dependence over the condition number of inner problem. Meanwhile, there are also works utilizing other strategies like penalty methods~\cite{mehra2019penalty}, and also other formulations like the case where the inner problem has non-unique minimizers~\cite{li2020improved}. 

The bilevel optimization has been widely applied to various machine learning applications. Hyper-parameter optimization~\cite{lorraine2018stochastic, okuno2018hyperparameter, franceschi2018bilevel} uses bilevel optimization extensively. Besides, the idea of bilevel optimization has also been applied to meta learning~\cite{zintgraf2019fast, song2019maml, soh2020meta}, neural architecture search~\cite{liu2018darts, wong2018transfer, xu2019pc}, adversarial learning~\cite{tian2020alphagan, yin2020meta, gao2020adversarialnas}, deep reinforcement learning~\cite{yang2018convergent, tschiatschek2019learner}, \emph{etc.} For a more thorough review of these applications, please refer to the Table~2 of the survey paper by~\citet{liu2021investigating}.


\section{A General Formulation of Hyper-Gradient Approximation}
\label{sec:geeral_form}
In general, bilevel optimization has the following form:
\begin{equation}
\label{eq:def}
\begin{split}
\underset{\lambda \in \Lambda}{\min}\ \ f(\lambda) \coloneqq F(\lambda, \omega_{\lambda})\quad  \emph{s.t.}\; \omega_{\lambda} = \underset{\omega}{\arg\min} \ G(\lambda, \omega)
\end{split}
\end{equation}
where $F$, $G$ denote the outer and inner problems, $\lambda$, $\omega$ denote the outer and inner variables. Under mild assumptions, the hyper-gradient $\nabla_{\lambda} f$ can be expressed in Proposition~\ref{lemma: exact-hg}:
\begin{proposition}
\label{lemma: exact-hg}
If for any $\lambda \in \Lambda$, $\omega_{\lambda}$ is unique, and $\partial_{\omega^2} G(\lambda, \omega_\lambda)$ is invertible, we have:
\begin{equation}
\label{eq:hp-grad}
\begin{split}
\nabla_{\lambda} f = \partial_{\lambda} F(\lambda, \omega_{\lambda}) + \nabla_{\lambda} \omega_{\lambda} \times \partial_{\omega}F(\lambda, \omega_{\lambda})
\end{split}
\end{equation}
and $\nabla_{\lambda} \omega_{\lambda} = -\partial_{\omega\lambda}G(\lambda, \omega_{\lambda}) \partial_{\omega^2} G(\lambda, \omega_\lambda)^{-1}$.
\end{proposition}
\noindent The proof of this proposition is a direct application of the implicit function theorem (we include it in Appendix~B.2 for completion). Eq.~(\ref{eq:hp-grad}) is hard to evaluate due to the involved matrix inversion, and we instead evaluate it approximately. Various approximate methods are proposed in the literature. The most well-known ones are: Back Propagation through Time (BP), Neumann series (NS) and conjugate gradient descent (CG). They look quite different on the first sight: BP is derived based on the chain rule, NS and CG are based on different ways of approximating $\partial_{\omega^2} G(\lambda, \omega_\lambda)^{-1}$. However, they share a common structure, as stated in the following lemma. We use $\nabla_{\lambda} f$ to represent $\nabla_{\lambda} f(\lambda)$ when the outer state $\lambda$ is clear from the context. In the remainder of this section: we use $[K]$ to denote the sequence, $k$ refers to $k_{th}$ element of $[K]$, while $s$ refers to $s_{th}$ element of sub-sequence $[k]$.
\begin{lemma}
\label{lemma:approx-hg}
Given a positive integer $K$, a sequence of inner variable states $\{\omega_{k}\}$, vectors $\{p_k\}$ and a sequence of coefficients $\{\beta_k\}$ for $k\in[K]$. A general form of approximate hyper-gradient evaluated at state $\lambda$ is: \[\nabla_{\lambda} f_K = \partial_{\lambda} F(\lambda, \omega_{K}) - \sum_{k=0}^{K-1} \beta_k s_k\] where $s_k$ has two modes:
\begin{equation*}
\begin{split}
s_k = \bigg\{\begin{array}{l}
      \partial_{\omega\lambda} G(\lambda, \omega_k)\prod_{s={k+1}}^{K-1}(I - \beta_s\partial_{\omega^2} G(\lambda,\omega_{s}))p_K\\
      \partial_{\omega\lambda} G(\lambda, \omega_K)\prod_{s={k+1}}^{K-1}(I - \beta_s\partial_{\omega^2} G(\lambda,\omega_{s}))p_k
  \end{array}
\end{split}
\end{equation*}
We call these two modes as backward and forward respectively (in terms of $p_k$).
More specifically, we have:
\begin{enumerate}
    \item BP is in backward mode with $\omega_{k} = \hat{\omega}_k$, $p_k = \partial_{\omega} F(\lambda, \hat{\omega}_k)$ and $\beta_k = \eta_k$ for $k \in [K]$. $\{\hat{\omega}_k\}$ are an inner variable sequence generated by the gradient descent algorithm and $\{\eta_k\}$ is the corresponding learning rate sequence;
    \item NS is in backward mode with $\omega_{k} = \hat{\omega}$, $p_k = \partial_{\omega} F(\lambda, \hat{\omega})$ and $\beta_k = \beta$ for $k \in [K]$. $\hat{\omega}$ is an inner variable state and $\beta$ is some constant;
    \item CG is in the forward mode with $\omega_{k} = \hat{\omega}$ for $k \in [K]$. $\hat{\omega}$ is an inner variable state, $\{p_k\}$ and $\{\beta_k\}$ is chosen adaptively by the conjugate steps.
\end{enumerate}
\end{lemma}

\begin{proof}
In the proof, we justify that the three cases mentioned above indeed satisfy the proposed general hyper-gradient computation formulation.

\noindent\textit{Case (1)}: Without loss of generality, assume we run a gradient descent optimizer $K$ steps over the inner problem in the BP method. In other words, it solves:
\begin{equation*}
\begin{split}
\underset{\lambda \in \Lambda}{\min}\ \ f_K(\lambda) &\coloneqq F(\lambda, \hat{\omega}_{K}) \\
\emph{s.t.}\ \hat{\omega}_{k} &= \hat{\omega}_{k-1} - \eta_k \nabla G(\lambda, \hat{\omega}_{k-1}), k \in [K]
\end{split}
\end{equation*}
where $\eta_k$ is the learning rate. By the chain rule, we get the hyper-gradient $\nabla_{\lambda} f_{K}$ of this problem:
\begin{equation*}
\begin{split}
\nabla_{\lambda} f_K =& \partial_{\lambda} F(\lambda, \hat{\omega}_K) - \sum_{k=0}^{K-1} \bigg(\eta_k\partial_{\omega\lambda} G(\lambda, \hat{\omega}_k)\times\\
&\prod_{s=k+1}^{K-1}(I - \eta_k\partial_{\omega^2}G (\lambda, \hat{\omega}_{s}))\bigg)\partial_{\omega} F(\lambda, \hat{\omega}_K)
\end{split}
\end{equation*}
It is easy to verify the claim in the lemma;

\noindent\textit{Case (2)}: The NS method is based on the hyper-gradient expression in Proposition~\ref{lemma: exact-hg}, but it assumes access of an approximate solution $\hat{\omega}$ instead of the optimum $\omega_{\lambda}$ and then approximates $\partial_{\omega^2} G(\lambda, \hat{\omega})^{-1}$ with the first $K$ terms of the Neumann series. More precisely:
\begin{equation}
\label{eq:neumann}
\begin{split}
\partial_{\omega^2} G(\lambda, \hat{\omega})^{-1} &\approx  \beta \sum_{k=0}^{K-1}\bigg(I - \beta\partial_{\omega^2} G(\lambda, \hat{\omega})\bigg)^k
\end{split}
\end{equation}
where $\beta$ is a small constant. Then we replace $\omega_{\lambda}$ with $\hat{\omega}$ in Eq.~(\ref{eq:hp-grad}) and combine with Eq.~(\ref{eq:neumann}). We have:
\begin{equation*}
\begin{split}
\nabla_{\lambda} f_K =& \partial_{\lambda} F(\lambda, \hat{\omega}) - \sum_{k=0}^{K-1}\bigg(\beta\partial_{\omega\lambda}G(\lambda, \hat{\omega})\times\\
&\bigg(I - \beta\partial_{\omega^2} G(\lambda, \hat{\omega})\bigg)^k\bigg) \partial_{\omega} F(\lambda, \hat{\omega})
 \end{split}
\end{equation*}
Substitute $k$ with $K-k-1$, it is straightforward to verify the claim in the lemma;

\noindent\textit{Case (3)}: The CG method uses the fact that $x = \partial_{\omega^2} G(\lambda, \hat{\omega})^{-1}\partial_{\omega}F(\lambda, \hat{\omega})$ is the minimizer of the quadratic optimization problem: $\underset{v}{\arg\min} \frac{1}{2}x^TAx - x^Tb$, where $A = \partial_{\omega^2} G(\lambda, \hat{\omega})$ and $b = \partial_{\omega}F(\lambda, \hat{\omega})$. Solving this quadratic problem with the conjugate gradient descent, we have the following update rule: 
\begin{equation*}
x_{k+1} = x_{k} - \alpha_k (p_k) = (I - \alpha_k A)x_{k} + \alpha_k (b +\gamma_k p_{k-1})
\end{equation*}
Then second equality follows the update rule of linear CG algorithm. $\alpha_k, \gamma_k$ are the learning rate and $p_{k}$ is the conjugate directions. Please refer to section 5 by~\citeauthor{nocedal2006numerical} for more details. Then $x_K$ takes the following form:
\begin{equation}
\label{eq:approx-cg}
\begin{split}
x_{K} =& \sum_{k=0}^{K-1}\alpha_k\prod_{s={k+1}}^{K-1}(I - \alpha_s A)(b + \gamma_k p_{k-1})\\
\end{split}
\end{equation}
Substitute the values of $A$ and $b$ into Eq.~(\ref{eq:approx-cg}) and then combine it with Eq.~(\ref{eq:hp-grad}), where we approximate $\partial_{\omega\lambda} G(\lambda, \omega_{\lambda})$ with $\partial_{\omega\lambda} G(\lambda, \hat{\omega})$.
It is straightforward to verify the claim in the lemma. This completes the proof.
\end{proof}

The general formulation in the lemma provides a unified view of the BP, NS and CG methods. 
Firstly, we can verify that using the NS method is equivalent to solving the quadratic problem (defined in case (3)) with (constant learning rate) the gradient descent. Since the CG method usually performs better than gradient descent in solving linear systems, we expect the CG method requires smaller $K$ than the NS method to reach a given estimation error. Then we compare NS with BP, their difference lies in the sequence $\{\omega_k\}$ and $\{\beta_k\}$: NS uses a single state $\hat{\omega}$($\beta$) , while BP uses a sequence of states $\{\omega_k\}$($\{\beta_k\}$).

It would be interesting to identify sufficient conditions for $\{\beta_k\}$, $\{\omega_k\}$ and $\{p_k\}$ such that the general formulation converges to the exact hyper-gradient $\nabla_{\lambda} f$. In fact, we have the following lemma:
\begin{lemma}
\label{lemma:conv_cond_main}
Suppose we denote $m_k = \prod_{s=0}^{k}(1 - \mu_G\beta_s)$, $e_{\omega, k} = ||\omega_k - \omega_{\lambda}||$, $e_{p,k} = ||p_k - \partial_{\omega} F(\lambda, \omega_{\lambda})||$ for $k\in [K]$. Then if $\underset{K\to\infty}{\lim} m_K = 0$, $\underset{k\to\infty}{\lim} e_{\omega,K} = 0$, $\underset{K\to\infty}{\lim} m_K\sum_{k=0}^{K-1}\beta_k e_{\omega,k}/m_k$ is finite, in addition, for the backward mode: $\underset{K\to\infty}{\lim} e_{p,K} = 0$; for the forward mode: $\underset{K\to\infty}{\lim}  m_K\sum_{k=0}^{K-1}\beta_k e_{p,k}/m_k$ is finite.
Then we have $\nabla_{\lambda} f_K \to \nabla_{\lambda} f$, when $K\to\infty$, where $\nabla_{\lambda} f_K$ is defined in Lemma~\ref{lemma:approx-hg}.
\end{lemma}
We defer the proof of Lemma~\ref{lemma:conv_cond_main} in in Appendix~C, we show some basic ideas here. Firstly, it is straightforward to show $\partial_{\lambda} F(\lambda, \omega_{K}) \to \partial_{\lambda} F(\lambda, \omega_{\lambda})$ by using the smoothness assumption and the condition $\omega_{K} \to \omega_{\lambda}$. For the term $\sum_{k=0}^{K-1} \beta_k s_k$, we take the backward mode as an example (the forward mode follows similar idea). We denote $A_K$ and $A^*$ as:
\begin{equation}
\label{eq:backward_recur}
	A_K = \sum_{k=0}^{K-1}\beta_k\partial_{\omega\lambda} G(\lambda, \omega_{k})\prod_{s={k+1}}^{K-1}\left(I - \beta_s\partial_{\omega^2} G(\lambda,\omega_s)\right)
\end{equation}
and $A^* = \partial_{\omega\lambda}G(\lambda, \omega_{\lambda}) \partial_{\omega^2} G(\lambda, \omega_\lambda)^{-1}$. To show $A_K \to A^*$, we use the recursive relation of $A_k$ and $A^*$:
\begin{equation}
\label{eq:recursive}
\begin{split}
A_{k+1} =& A_k\left(I - \beta_k\partial_{\omega^2} G(\lambda,\omega_k)\right) + \beta_k\partial_{\omega\lambda} (\lambda, \omega_{k})\\
=& (1 -\beta_k)A_k + \beta_k(A_k( I  -\partial_{\omega^2} G(\lambda,\omega_k))\\
&+ \partial_{\omega\lambda} G(\lambda, \omega_{k}))\\
\end{split}
\end{equation}
and $A^{*} = A^{*}(I - \partial_{\omega^2} G(\lambda,\omega_{\lambda})) + \partial_{\omega\lambda} G(\lambda, \omega_{\lambda})$.
Based on the recursive relation above and some linear algebra derivation, we get:
\begin{equation}
\label{eq:progress}
\begin{split}
||A_{k+1} - A^{*}||
\le& (1- \mu_G\beta_k)||A_k - A^*|| + C_{1}\beta_k e_{\omega,k}
\end{split}
\end{equation}
It is straightforward to verify that the conditions in the lemma guarantee the convergence of Eq.~(\ref{eq:progress}). As shown in Eq.~(\ref{eq:progress}), the momentum coefficient $\beta_k$ represents the trade-off between the progress and the induced error in one iteration: Larger $\beta_k$ leads to better contraction factor $\epsilon$, but also bigger bias term $e_{\omega,k}$. In fact, we can get meaningful convergence rate of $\nabla_{\lambda} f_K$ for several special choices of sequences. As shown in Corollary~\ref{corollary:choice1_main} and Corollary~\ref{corollary:choice2_main}:
\begin{corollary}
\label{corollary:choice1_main}
Given a positive integer $K$, inner state $\hat{\omega}_K$ and constant $\beta$. Then we set $\omega_{k} = \hat{\omega}_K$, $p_k = \partial_{\omega} F(\lambda, \hat{\omega}_K)$ and $\beta_k = \beta$ for $k \in [K]$. If $\exists\ \epsilon \in (0,1)$, such that $\beta_k \in (\epsilon/\mu_G, 1/\mu_G)$, then we have:
\begin{equation*}
   ||\nabla_{\lambda} f_K - \nabla_{\lambda} f|| = O(e_{\omega,K})
\end{equation*}
\end{corollary}

\begin{corollary}
\label{corollary:choice2_main}
Given a positive integer $K$, we have sequence $\{\hat{\omega}_k\}$ and $\{\hat{\beta}_k\}$ for $k\in[K]$. We pick $\omega_{k} = \hat{\omega}_k$, $p_k = \partial_{\omega} F(\lambda, \hat{\omega}_K)$ and $\beta_k = \hat{\beta}_k$ for $k \in [K]$. Suppose we have $\beta_k = O(k^{-1})$ and $e_{\omega,k} =O(k^{-0.5})$, then:
\begin{equation*}
   ||\nabla_{\lambda} f_K - \nabla_{\lambda} f|| = O(K^{-{0.5}})
\end{equation*}
\end{corollary}

In fact, the sequences chosen in  Corollary~\ref{corollary:choice1_main} correspond to that in the NS method, and we show that if we choose $\beta$ properly, the hyper-gradient estimation converges in the rate of $e_{\omega, K}$. The BP method corresponds to Corollary~\ref{corollary:choice2_main}. Suppose we optimize the inner problem with stochastic gradient descent and learning rate $\beta_k = O(1/k)$, we get $e_{\omega,k} = O(1/\sqrt{k})$. This satisfies the condition in the corollary. 
To the best of our knowledge, this is first complexity result of the stochastic case. \citeauthor{grazzi2020iteration} considered the linear convergence case for BP method. We introduce some more examples of the application of Lemma~\ref{lemma:conv_cond_main}  in the Appendix~C.

\section{New Fully Single Loop Algorithm (FSLA)}
In this section, we introduce a new single loop algorithm for bilevel optimization. In section~\ref{sec:geeral_form}, we identify a general formulation of hyper-gradient approximation in Lemma~\ref{lemma:approx-hg}: $\nabla_{\lambda} f_K =\partial_{\lambda} F(\lambda, \omega_{K}) - \sum_{k=0}^{K-1} \beta_k s_k$, where $s_k$ has two modes: forward mode and backward mode. For both modes, they can be expressed by a recursive equation. In the backward mode, we write a recursive relation in terms of $A_k$ as defined in Eq.~(\ref{eq:backward_recur}) and Eq.~(\ref{eq:recursive}). Similarly, we define $v_K = \sum_{k=0}^{K-1}\beta_k\prod_{s={k+1}}^{K-1}\left(I - \beta_s\partial_{\omega^2} G(\lambda,\omega_s)\right)p_k$ in the forward mode, and derive a recursive relation as:
\begin{equation}
\label{eq:recursive_forward}
\begin{split}
v_{k+1} =& \left(I - \beta_k\partial_{\omega^2} G(\lambda,\omega_k)\right)v_k + \beta_k p_k\\
\end{split}
\end{equation}
Based on this observation, we can propose a new single loop bilevel optimization algorithm without Hessian Inverse. In previous literature, researchers maintain a inner state $\omega_k$ to avoid the inner loop solving $\omega_{\lambda}$ for each new outer state $\lambda$. On top of this, we maintain a new state $v_k$ and evaluate the hyper-gradient as follows:
\begin{equation}
\label{eq:recur-eq}
\begin{split}
v_{k} &= \beta_k\partial_{\omega} F(\lambda, \omega_k) + (I - \beta_k\partial_{\omega^2} G(\lambda,\omega_k))v_{k-1}\\
\nabla_{\lambda} f_k &= \partial_{\lambda} F(\lambda, \omega_{k}) - \partial_{\omega\lambda} G(\lambda, \omega_{k})v_{k}\\
\end{split}
\end{equation}
Note we set $p_k = \partial_{\omega} F(\lambda, \omega_k)$. Then we alternatively update $\omega_k$, $v_k$ and $\lambda_k$, and achieve  a fully single loop algorithm without Hessian-Inverse. Note that it is also possible to keep track of $A_k$, but then we need to store a matrix, which cost more storage. More formally, we get the following new alternative update rule:
\begin{equation}
\label{eq:single-loop}
\begin{split}
\lambda_{k} &= \lambda_{k-1} - \alpha_{k-1}\nabla_{\lambda} f_{k-1}\\
\omega_{k} &= \omega_{k-1} - \tau_k \partial_{\omega}G(\lambda_{k}, \omega_{k-1})\\
v_{k} &= \beta_k\partial_{\omega} F(\lambda_{k}, \omega_{k}) + (I - \beta_k\partial_{\omega^2} G(\lambda_{k},\omega_{k}))v_{k-1}\\
\nabla_{\lambda} f_k &= \partial_{\lambda} F(\lambda_{k}, \omega_{k}) - \partial_{\omega\lambda} G(\lambda_{k}, \omega_{k})v_{k}
\end{split}
\end{equation}
where $\tau_k$ and $\alpha_k$ are learning rates for inner and outer updates. As a comparison, existing single loop algorithms recompute $\nabla_{\lambda} f$ from scratch for each new $\lambda$, while our algorithm performs one step update over $v_k$. What's more, we can express $\nabla_{\lambda} f_k$ in Eq.~(\ref{eq:single-loop}) as follows:
\begin{equation*}
    \begin{split}
        \nabla_{\lambda} f_{K} (\lambda_K) =& \partial_{\lambda} F(\lambda_{K}, \omega_{K}) - \sum_{k=0}^{K-1} \beta_k       \partial_{\omega\lambda} G(\lambda_{K}, \omega_{K}) \times\\
        &\prod_{s={k+1}}^{K-1}(I - \beta_s\partial_{\omega^2} G(\lambda_{s},\omega_{s}))\partial_{\omega} F(\lambda_{k}, \omega_k)
    \end{split}
\end{equation*}
This almost fits the forward mode of the general formulation in Lemma~\ref{lemma:approx-hg} except that the outer state is a sequence $\{\lambda_k\}$. 
In Lemma~\ref{lemma:conv_cond_main}, we show that $\omega_{K} \to \omega_{\lambda}$ is a sufficient condition of $\nabla_{\lambda} f_K(\lambda) \to \nabla_{\lambda} f(\lambda)$. It is reasonable to guess that if $\lambda_{K} \to \lambda$ and $\omega_{K} \to \omega_{\lambda}$, the convergence is also guaranteed. 
But the analysis is more challenging than Lemma~\ref{lemma:conv_cond_main} as the three terms $\lambda_k$, $\omega_k$ and $\nabla_{\lambda} f_k$ entangle with each other. However, we will show in the next section that $\lambda_K \to \lambda^*$ when $K \to \infty$, where $\lambda^*$ is the optimal point of the outer function $f(\lambda)$.

Finally in Algorithm~\ref{alg:example}, we provide the pseudo code of our fully single loop algorithm. Compared to Eq.~(\ref{eq:single-loop}), we assume access of the stochastic estimate of related values. What's more, we also maintain a momentum of the hyper-gradient $d_k$ with variance reduction correction in Line 10. This term is used to control the stochastic noise. The momentum-based variance reduction technique is recently widely used in the single level stochastic optimization, such as STORM~\cite{cutkosky2019momentum}. 

\begin{algorithm}[tb]
	\caption{Fully Single Loop Bilevel Optimization Algorithm (FSLA)}
	\label{alg:example}
	\begin{algorithmic}[1]
		\STATE {\bfseries Input:} Initial state $\lambda_0 \in \Lambda$, $\omega_0 \in R^n$; the number of hyper-iterations $K$; constants $c_{\tau}$, $c_{\beta}$, $c_{\eta}$, $\delta$
		\FOR{$k\leftarrow0$ {\bfseries to} $K - 1$}
		\STATE $\alpha_k \leftarrow \delta/\sqrt{k}$, $\lambda_{k+1} \leftarrow \lambda_k - \alpha_k d_k$
		\STATE $\tau_{k+1} \leftarrow c_{\tau}\alpha_k$, $\beta_{k+1} \leftarrow c_{\beta}\alpha_k$ and $\eta_{k+1} \leftarrow c_{\eta}\alpha_k$
		\STATE Sample $\xi_{k+1} (\xi_{k+1,1}-\xi_{k+1,5})$
		\STATE $\omega_{k+1} \leftarrow \omega_k - \tau_{k+1} \partial_{\omega} G(\lambda_{k+1}, \omega_k; \xi_{k+1,1})$
		\STATE $v_{k+1} \leftarrow \beta_{k+1}\partial_{\omega} F(\lambda_{k+1}, \omega_{k}; \xi_{k+1,2}) +  (I - \beta_{k+1}\partial_{\omega^2} G(\lambda_{k+1},\omega_k; \xi_{k+1,3}))v_k$
		\STATE $\nabla f_{k+1}(\xi_{k+1}) \leftarrow \partial_{\lambda} F(\lambda_{k+1}, \omega_{k+1}; \xi_{k+1,4}) - \partial_{\omega\lambda} G(\lambda_{k+1}, \omega_{k+1}; \xi_{k+1,5})v_{k+1}$
		\STATE $d_{k+1} \leftarrow \nabla f_{k+1}(\xi_{k+1}) + (1 - \eta_{k+1}) (d_{k} - \nabla f_{k} (\xi_{k+1}))$; 
		\ENDFOR
	\end{algorithmic}
\end{algorithm}

\section{Theoretical Analysis}
In this section, we prove the convergence of our proposed single loop bilevel optimization algorithm. We consider the non-convex-strongly-convex case. We first briefly state some assumptions needed in our theoretical analysis for the convenience of discussion. A formal description of the assumptions is in the Appendix~A.
\begin{assumption}
	\label{outer_assumption} (Outer Function)
	Function $F$ is possibly non-convex, Lipschitz continuous with constant $L_{F,\lambda}$ (\emph{w.r.t $\lambda$}) and $L_{F, \omega}$ (\emph{w.r.t $\omega$}), and has bounded gradient with constant $C_F$;
\end{assumption}

\begin{assumption}
	\label{inner_assumption} (Inner Function)
	Function G is continuously twice differentiable, $\mu_G$-strongly convex w.r.t $\omega$ for any given $\lambda$, Lipschitz continuous with constant $L_{G,\omega}$ (\emph{w.r.t $\omega$}). For higher-order derivatives, we have:
	\begin{itemize}
		\item[a)] $\|\nabla_{\omega\lambda}^2 G(\lambda, \omega)\| \le C_{G,\omega\lambda}$ for some constant $C_{G,\omega\lambda}$
		\item[b)] $\nabla_{\omega\lambda}^2 G(\lambda,\omega)$ and $\nabla_{\omega}^2 G(\lambda,\omega)$ are Lipschitz continuous with constant $L_{G,\omega\lambda}$ and $L_{G,\omega\omega}$ respectively
	\end{itemize}
\end{assumption}
\begin{assumption}
	\label{noise_assumption} (Bounded Variance)
	We have an unbiased stochastic oracle with bounded variance for estimating the related properties (gradient and Hessian), \emph{e.g.} $E[\partial_\lambda F(\lambda, \omega; \xi)] = \partial_\lambda F(\lambda, \omega)$ and $var(\partial_\lambda F(\lambda, \omega; \xi)) \le \sigma^2$
\end{assumption}

We first bound the one iteration progress in the following lemma, which follows the usage of smoothness assumption.
\begin{lemma}
	\label{lemma:one-iter}
	Under Assumptions A, B, C, we have:
	\begin{equation*}
	\begin{split}
	&E[f(\lambda_{k+1})]\\
	\le& f(\lambda_k) - \frac{\alpha_k}{2}||\nabla f(\lambda_k)||^2 + \alpha_k\Gamma_2^2E[||\omega_{k} - \omega_{\lambda_k}||^2]\\
	&+ 4\alpha_k\Gamma_1^2E[||v_{k} - v_{\lambda_{k}}||^2] + \alpha_k E[||\nabla f_k -  d_k||^2]\\
	&- \frac{\alpha_k}{2}(1 - \alpha_k L_f) E[||d_k||^2]\\
	\end{split}
	\end{equation*}
	where $v_{\lambda} = (I - \partial_{\omega} \Phi(\lambda, \omega_\lambda))^{-1} \partial_{\omega}F(\lambda, \omega_{\lambda})$ and $\Gamma_1^2 = C_{G,\omega\lambda}^2$, $\Gamma_2^2= 2L_{F,\lambda}^2 + 4C_{F,\omega}^2L_{G,\omega\lambda}^2/\mu_G^2$ are constants.
\end{lemma}
The proof is in Appendix~D.2. Lemma~\ref{lemma:one-iter} shows that there are three kinds of errors at each iteration: estimation error of $\omega_{\lambda_k}$ ($||\omega_{k+1} - \omega_{\lambda_k}||^2$), error of $v_{\lambda_k}$ ($||v_{k+1} - v_{\lambda_{k}}||^2$) and error of momentum $d_k$ ($||\nabla f_k -  d_k||^2$). We denote them as $A_k$, $B_k$ and $C_k$ in the remainder of this section. The three errors entangle with each other as shown by Line 7-10 of Algorithm~\ref{alg:example}, \emph{e.g.} the estimation error to $\omega_{\lambda_{k-1}}$ will contribute to the error of momentum $d_k$ as shown in Line 8. In fact, we can bound them with the following inequality (use $A_k$ as an example):
\[
A_k \le \beta A_{k-1} + C_1 B_{k-1} + C_2 C_{k-1} + C_3
\]
with $\beta < 1$ and $C_1$, $C_2$, $C_3$ are terms not relevant to $A_k$, $B_k$ and $C_k$. Please check the Appendix~D.3 for more details. Specially, the momentum term $C_k$ reduces the variance similarly to that in the single level variance-reduction optimizers, by which we mean:
\begin{equation*}
\begin{split}
&E[||d_k - \nabla f_k||^2]\\
\le& (1 - \eta_k)^2E[||d_{k-1} - \nabla f_{k-1}||^2] + 2\eta_k^2\sigma^2\\
&+ 2(1 - \eta_k)^2 \underbrace{E[||\nabla f_k(\xi_k) - \nabla f_{k-1}(\xi_k)||^2]}_{\Pi}\\
\end{split}
\end{equation*}
The part of noise proportional to $(1 - \eta_k^2)\sigma^2$ is absorbed in the term $\Pi$ due to the correction made by $\nabla f_{k-1}(\xi_k)$ in the $d_k$ update rule. However, the key difference is that $\Pi$ not only relies on $E[||d_{k-1}|||^2]$ but also $A_{k-1}$ and $B_{k-1}$. 

Finally, to show the the convergence of Algorithm~\ref{alg:example}, we denote the following potential function:
\begin{equation*}
\begin{split}
\Phi_k =& f(\lambda_k) + D_1A_k + D_2B_k + D_3C_k
\end{split}
\end{equation*}
where $D_1$, $D_2$ and $D_3$ are some constants. Combine Lemma~\ref{lemma:one-iter} with the inequalities for $A_k$, $B_k$ and $C_k$, and we have: $\Phi_{k+1} - \Phi_k \le -\alpha_k/2||\nabla f(\lambda_k)||^2 + C\alpha_k^2\sigma^2$ where $C$ is some constant. With the above inequality, we can get a convergence rate of $O(1/\sqrt{K})$ ($O(1/\epsilon^2)$) by choosing the learning rate with $O(1/\sqrt{k})$. More formally, we have:
\begin{theorem}
	With Assumption A, B, C hold, and take $\beta_k = c_{\beta}\alpha_k$, $\tau_k = c_{\tau}\alpha_k$, $\eta_k = c_{\eta}\alpha_k$, and $\alpha_k = \frac{\delta}{\sqrt{k}}$. We have:
	\begin{equation*}
	\begin{split}
	\frac{1}{K}\sum_{k=0}^{K-1}||\nabla f(\lambda_k)||^2
	\le \frac{2\Phi_{0}}{\delta\sqrt{K}} + \frac{2\delta\bar{C}\sigma^2}{\sqrt{K}}\\
	\end{split}
	\end{equation*}
	where $\bar{C}$, $\delta$, $c_{\beta}$, $c_{\tau}$ and $c_{\eta}$ are some constants.
\end{theorem}
\noindent The proof of the theorem is included in Appendix~D.4.
\section{Experiments}
In this section, we perform experiments to empirically verify the effectiveness of our algorithm. We first perform experiments over a quadratic objective with synthetic dataset to validate Lemma~\ref{lemma:conv_cond_main}. Then we perform a common benchmark task in bilevel optimization: data hyper-cleaning. The experiments are run over a machine with Intel Xeon E5-2683 CPU and 4 Nvidia Tesla P40 GPUs.

\subsection{Synthetic Dataset: Quadratic Objective}
In this experiment, we verify Lemma~\ref{lemma:conv_cond_main} over some synthetic data. We consider the bilevel optimization problem where both outer and inner problems are quadratic. To make it simpler, the outer problem does not depend on the outer variable directly.  More precisely, we study:
\begin{equation*}
\begin{split}
\underset{\lambda \in \Lambda}{\min}\ f(\lambda) &\coloneqq ||A_{o}\omega_{\lambda} - b_{o}||^2 \\
s.t.\ \omega_{\lambda} &=\underset{\omega}{\arg\min} ||A_{i, \lambda}\lambda + A_{i, \omega}\omega - b_{i}||^2
\end{split}
\end{equation*}
For this bilevel problem, we can solve the exact minimizer of the inner problem and then evaluate the exact hyper-gradient based on the Proposition~\ref{lemma: exact-hg}. This makes it easier to compute and compare the approximation error of different methods. In experiments, we pick problem dimension 5 and randomly sample 10000 data points. We construct the dataset as follows: first randomly sample  $A_o, A_{i,\lambda}, A_{i, \omega} \in \mathbb{R}^{10^4\times 5} $ and $\lambda, \omega, \omega_{\lambda} \in \mathbb{R}^{5}$ from the Uniform distribution, then we construct $b_o$ and $b_i$ by $A_{o}\omega_{\lambda} + \sigma_o$ and $A_{i, \lambda}\lambda + A_{i, \omega}\omega + \sigma_i$, where $\sigma_i$ and $\sigma_o$ are Gaussian noise with mean zero and variance $0.1$.  We use this simple task to validate our claim in Lemma~\ref{lemma:conv_cond_main}. More precisely, we fix the outer state $\lambda$ and estimate the hyper-gradient with different sequence $\{\omega_k\}, \{\beta_k\}, \{p_k\}$ and then compare their estimation errors. The results are shown in Figure~\ref{fig:toy}.
\begin{figure}[ht]
	\begin{center}
		\includegraphics[width=0.6\columnwidth]{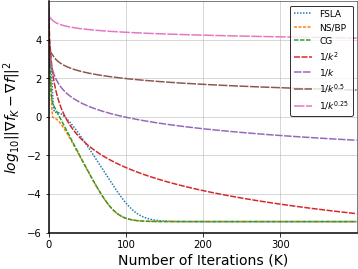}
		\caption{The estimation error of hyper-gradient $||\nabla f_K - \nabla f||^2$ for different sequences $\{\omega_k\}$.}
		\label{fig:toy}
	\end{center}
\end{figure}

We perform two sets of experiments. The first set includes our single loop method FSLA, and NS, BP and CG, which are three cases discussed in Lemma~\ref{lemma:approx-hg}. For these methods, we generate the sequence $\{\omega_k\}$ through solving the inner problem with $K$ steps of gradient descent and we use learning rate $\beta_k$ we pick $2\times 10^{-5}$. Note since the second order derivatives of the quadratic objective are constant, BP and NS are the same. As shown by the figure, the hyper-gradient estimation errors of all the four methods converge. Our FSLA takes a bit more number of iterations to converge, however, our algorithm takes less running time. Our method requires $O(1)$ matrix-vector query for every given $K$, while the other methods requires $O(K)$ queries. In the next set of experiments, we compare with some synthetic $\{\omega_k\}$ sequences which have different convergence rate. More precisely, we use sequence $\{\omega + \tilde{\omega}/(k^{\alpha})\}$, where $\tilde{\omega}$ is some random start point, we pick different alpha values (2, 1, 0.5, 0.25). We estimate $\nabla_{\lambda} f$ according to Corollary~\ref{corollary:choice1_main} (same as the NS), the learning rate is chosen as $2\times 10^{-5}$. Corollary~\ref{corollary:choice1_main} bounds the convergence rate of $\nabla_{\lambda} f_K$ with the rate of $\omega_k$, which is well verified through the results shown in Figure~\ref{fig:toy}.

\subsection{Hyper Data-cleaning}
\begin{figure}[ht]
	\begin{center}
		\includegraphics[width=0.48\columnwidth]{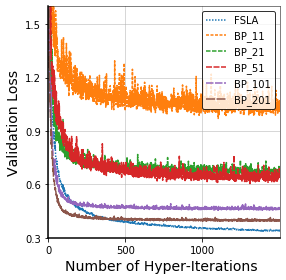}
		\includegraphics[width=0.48\columnwidth]{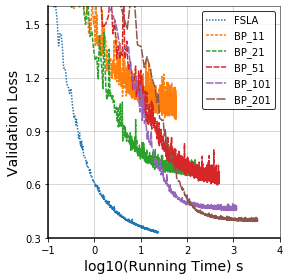}
	\end{center}
	\caption{FSLA \emph{vs} BP plot of Validation Loss w.r.t Number of hyper-iterations (Left) and $log10$(Running Time) (Right). The perturbation rate $\gamma$ is 0.8. The post-fix of legend represents the number of inner iterations $T$.}
	\label{fig:compare-bp}
\end{figure}

\begin{figure*}[t]
	\begin{center}
		\includegraphics[width=0.48\columnwidth]{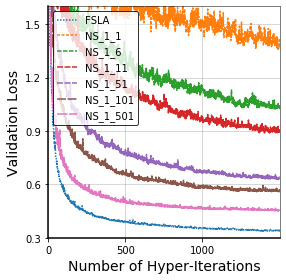}
		\includegraphics[width=0.48\columnwidth]{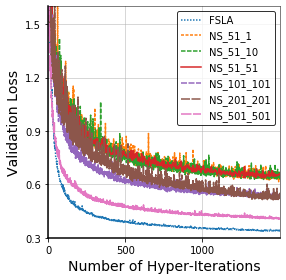}
		\includegraphics[width=0.48\columnwidth]{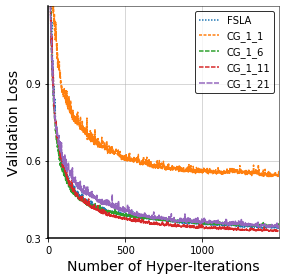}
		\includegraphics[width=0.48\columnwidth]{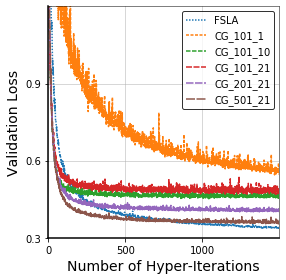}
		\includegraphics[width=0.48\columnwidth]{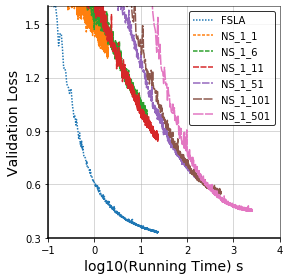}
		\includegraphics[width=0.48\columnwidth]{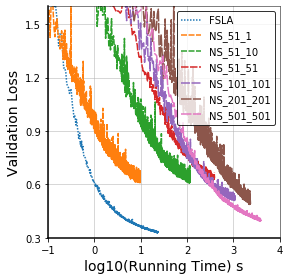}
		\includegraphics[width=0.48\columnwidth]{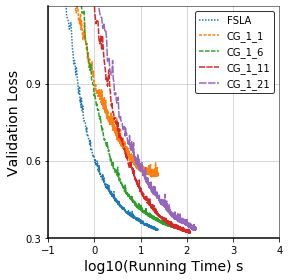}
		\includegraphics[width=0.48\columnwidth]{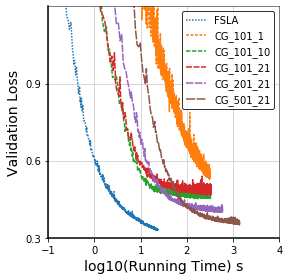}
	\end{center}
	\caption{FSLA \emph{vs} NS and CG plot of Validation Loss w.r.t Number of hyper-iterations (Top) and $log10$(Running Time) (Bottom). The perturbation rate $\gamma$ is 0.8. The post-fix of legend represents the number of inner iterations $T$ and approximate steps $K$.If inner gradient steps equal to 1, we use the warm start trick, otherwise not.}
	\label{fig:compare-ns}
\end{figure*}


In this experiment, we demonstrate the efficiency of our FSLA, especially the effect of tracking hyper-gradient history with $v_k$. More precisely, we compare with three hyper-gradient evaluation methods: BP, NS and CG. For NS and CG methods, we consider both the double loop version and the single loop version. In the single loop version, we update the inner variable with the warm start strategy.

Data cleaning denotes the task of cleaning a noisy dataset. Suppose there is a noisy dataset $D_i$ with $N_i$ samples (the label of some samples are corrupted), the aim of the task is to identify those corrupted data samples. Hyper Data-cleaning approaches this problem by learning a weight per sample. More precisely, we solve the following bilevel optimization problem:
\begin{equation*}
\begin{split}
&\underset{\lambda \in \Lambda}{\min}\ l(\omega_{\lambda}; D_{o})\ \emph{s.t.}\ \omega_{\lambda} = \underset{\omega\in \mathbb{R}^d}{\arg\min}\ \frac{1}{N_i}\sum_{j=1}^{N_i} \sigma(\lambda_j) l(\omega, D_{i,j})
\end{split}
\end{equation*}
In the inner problem, we minimize a weighted average loss $l$ over the training dataset $D_i$, with $\sigma(\lambda)$ the sample-wise weight ($\sigma(\cdot)$ is a normalization function and we use $Sigmoid$ in experiments), suppose the minimizer of the inner problem is $\omega_{\lambda}$. In in the outer problem, we evaluate $\omega_\lambda$ (a function of $\lambda$) over a validation dataset $D_o$. Then the bilevel problem will find $\lambda$ such that $\omega_{\lambda}$ is optimal as evaluated by the validation set.

More specifically, we perform this task over several datasets: MNIST~\cite{lecun2010mnist}, Fashion-MNIST~\cite{xiao2017fashion} and QMNIST~\cite{yadav2019cold}. We construct the datasets as follows: for the training set $D_i$, we choose 5000 images from the training set, and randomly perturb the label of $\gamma$ percentage of images. While for the validation set $D_o$, we randomly select another 5000 images but without any perturbation (all labels are correct). We adopt a 4-layer convolutional neural network in the training.
The experimental results are shown in Figure~\ref{fig:compare-bp} and Figure~\ref{fig:compare-ns}. For all the methods, we solve the inner problem with stochastic gradient descent, while for the outer optimizer, all the methods use the variance reduction for fair comparison. In experiments, we vary the number of inner gradient descent steps $T$ and the number of hyper-gradient approximation steps $K$. The legend in the figures has the form of \textit{method-$T$-$K$}. For other hyper-parameters, we perform grid search for each method and choose the best one (hyper-parameters selection is in Appendix~E).

\begin{figure}[ht]
	\begin{center}
		\includegraphics[width=0.48\columnwidth]{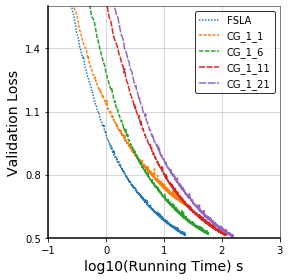}
		\includegraphics[width=0.48\columnwidth]{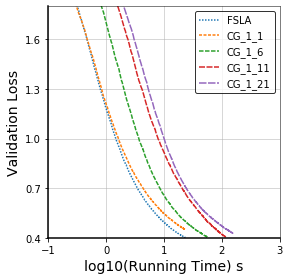}
	\end{center}
	\caption{FSLA \emph{vs} CG plot of Validation Loss w.r.t $log10$(Running Time). The Left plot shows Fashion-MNIST and the right plot shows the QMNIST. $\gamma$ is set 0.8.}
	\label{fig:compare-other}
\end{figure}

As shown by the figures, our method converges much faster than the baseline methods. Compared with BP, FSLA surpasses the best BP variant $BP_{201}$ at the 500 iteration, as for the running time, FSLA runs much faster. For NS and CG, $T=1$ in figures represents using the warm start, where the inner variable is updated from the state of last hyper-iteration.
The warm start trick has some kind of acceleration effects. However, for NS, the running time is dominated by the evaluation of Neumann Series. For CG, it converges with around 10 steps, but the extra time is still considerable compared to FSLA. Even with similar computation cost, FLSA still outperforms CG. \textit{CG\_1\_1} and FSLA both perform one step update per hyper-iteration, but CG converges much slower. This is due to failure of reusing the historical information. Finally, Figure~\ref{fig:compare-other} includes results for Fashion-MNIST and QMNIST, where we compare with the best baseline method CG. Our FLSA still outperforms it.

\section{Conclusion}
In this paper, we studied the bilevel optimization problem. More specifically, we first proposed a general formulation of hyper-gradient approximation. This formulation encompasses several important methods in the bilevel optimization. Then inspired by this, we introduced a new fully single loop algorithm, which performs alternative optimization of inner and outer variables. Our algorithm attains convergence rate $O(\epsilon^{-2})$. Moreover, the empirical results also verify the superior performance of our new algorithm.

\bibliography{aaai22}




\onecolumn
\appendix
\section{Some Mild Assumptions}
\label{appendix:A}
We first state the full assumptions needed in our analysis:
\begin{assumption}
\label{upper_assumption} (Outer function)
Function $F$ has the following properties:
\begin{itemize}
	\item[a)] $F(\lambda,\omega)$ is possibly non-convex, $\partial_{\lambda} F(\lambda,\omega)$ and $\partial_{\omega} F(\lambda,\omega)$ are Lipschitz continuous with constant $L_{F,\lambda}$ and $L_{F,\omega}$ respectively
	\item[b)] $\|\partial_{\lambda} F(\lambda, \omega)\| \le C_{F,\lambda}$ and $\|\partial_{\omega} F(\lambda, \omega)\| \le C_{F,\omega}$ for some constants $C_{F,\lambda}$ and $C_{F,\omega}$
\end{itemize}
\end{assumption}

\begin{assumption}
\label{lower_assumption} (Inner function)
Function G has the following properties:
\begin{itemize}
	\item[a)]  $G(x,y)$ is continuously twice differentiable, and $\mu_G$-strongly convex w.r.t $\omega$ for any given $\lambda$
	\item[b)] $\partial_{\omega} G(\lambda,\omega)$ is Lipschitz continuous with constant $L_{G, \omega}$, $\|\partial_{\omega\lambda}^2 G(\lambda, \omega)\| \le C_{G,\omega\lambda}$ for some constant $C_{G,\omega\lambda}$
	\item[c)] $\partial_{\omega\lambda}^2 G(\lambda,\omega)$ and $\partial_{\omega^2} G(\lambda,\omega)$ are Lipschitz continuous with constant $L_{G,\omega\lambda}$ and $L_{G,\omega\omega}$ respectively
\end{itemize}
\end{assumption}
\noindent For convenience, we assume Assumption A and Assumption B also hold for their stochastic version. Next we make the bounded stochastic noise assumption, \emph{i.e.}:
\begin{assumption}
\label{noise_assumption_append} (Bounded Variance)
We have access to a stochastic unbiased estimate of $\partial_{\omega^2}^2 G(\lambda,\omega)$, $\partial_{\omega\lambda}^2 G(\lambda,\omega)$, $\partial_\omega G(\lambda, \omega)$, $\partial_\omega F(\lambda, \omega)$, $\partial_\lambda F(\lambda, \omega)$ with their variance bounded by $\sigma^2$: 
\begin{itemize}
	\item[a)] $E[\partial_\lambda F(\lambda, \omega; \xi)] = \partial_\lambda F(\lambda, \omega)$ and $var(\partial_\lambda F(\lambda, \omega; \xi)) \le \sigma^2$
	\item[b)] $E[\partial_\omega F(\lambda, \omega; \xi)] = \partial_\omega F(\lambda, \omega)$ and $var(\partial_\omega F(\lambda, \omega; \xi)) \le \sigma^2$
	\item[c)] $E[\partial_\omega G(\lambda, \omega; \xi)] = \partial_\omega G(\lambda, \omega)$ and $var(\partial_\omega F(\lambda, \omega; \xi)) \le \sigma^2$
	\item[d)] $E[\partial_{\omega^2}^2 G(\lambda,\omega; \xi)] = \partial_{\omega^2}^2 G(\lambda,\omega)$ and $var(\partial_{\omega^2}^2 G(\lambda,\omega; \xi)) \le \sigma^2$
	\item[e)] $E[\partial_{\omega\lambda}^2 G(\lambda,\omega; \xi)] = \partial_{\omega\lambda}^2 G(\lambda,\omega)$ and $var(\partial_{\omega\lambda}^2 G(\lambda,\omega; \xi)) \le \sigma^2$ 
\end{itemize}
\end{assumption}

\section{Preliminary Results}
\label{Appendix:B}
\subsection{Some Basic Propositions}
\begin{proposition} (relaxed triangle inequality)
Let $\{x_k\}, k\in{K}$ be $K$ vectors. Then the following are true:
\begin{enumerate}
	\item $||x_i + x_j||^2 \le (1 + a)||x_i||^2 + (1 + \frac{1}{a})||x_j||^2$ for any $a > 0$, and
	\item  $||\sum_{k=1}^K x_k||^2 \le K\sum_{k=1}^{K} ||x_k||^2$
\end{enumerate}
\end{proposition}

\begin{proposition} (separating mean and variance)
Let $\{\Xi_k\}, k\in {K}$ be $K$ random variables. Suppose that $E(\Xi_i) = \xi_i$ and $Var(\Xi_i) \le \sigma^2$, then:
\[
E[||\sum_{i=1}^K \Xi_k||^2] \le ||\sum_{k=1}^K \xi_i|| + K^2\sigma^2
\]
\end{proposition}
Please refer to Lemma 3 and Lemma 4 in~\cite{karimireddy2020scaffold} for the proof of the two propositions.
\subsection{More Details about Proposition 1}
We first restate Proposition 1 as follows:
\begin{proposition}
If for any $\lambda \in \Lambda$, $\omega_{\lambda}$ is unique, and  $\partial_{\omega^2} G(\lambda, \omega_\lambda)$ is invertible, we have:
\begin{equation*}
\nabla_{\lambda} \omega_{\lambda} = -\partial_{\omega\lambda}G(\lambda, \omega_{\lambda})\partial_{\omega^2} G(\lambda, \omega_\lambda))^{-1}
\end{equation*}
and the hyper-gradient \emph{w.r.t} $\lambda$ has the following form:
\begin{equation*}
\nabla_{\lambda} f = \partial_{\lambda} F(\lambda, \omega_{\lambda}) - \partial_{\omega\lambda}G(\lambda, \omega_{\lambda})\partial_{\omega^2} G(\lambda, \omega_\lambda))^{-1} \partial_{\omega}F(\lambda, \omega_{\lambda})
\end{equation*}
\end{proposition}
\noindent This is a standard result in Bilevel optimization, we omit the proof here. Please refer to Lemma 2.1 in~\cite{ghadimi2018approximation} for more details.

\section{Proof for Lemma 3 and its Application}
\label{Appendix:C}
We first restate Lemma 3 as follows:
\begin{lemma} (Lemma 3 of main text)
\label{lemma:conv_cond}
Suppose we denote $m_k = \prod_{s=0}^{k}(1 - \mu_G\beta_s)$, $e_{\omega, k} = ||\omega_k - \omega_{\lambda}||$, $e_{p,k} = ||p_k - \partial_{\omega} F(\lambda, \omega_{\lambda})||$ for $k\in [K]$. Then if $\underset{K\to\infty}{\lim} m_K = 0$, $\underset{k\to\infty}{\lim} e_{\omega,K} = 0$, $\underset{K\to\infty}{\lim} m_K\sum_{k=0}^{K-1}\beta_k e_{\omega,k}/m_k$ is finite, in addition, for the backward mode: $\underset{K\to\infty}{\lim} e_{p,K} = 0$; for the forward mode: $\underset{K\to\infty}{\lim}  m_K\sum_{k=0}^{K-1}\beta_k e_{p,k}/m_k$ is finite.
Then we have $\nabla_{\lambda} f_K \to \nabla_{\lambda} f$, when $K\to\infty$, where $\nabla_{\lambda} f_K$ is defined in Lemma 2 of the main text.
\end{lemma}

\begin{proof}
Firstly, we prove for the backward mode case. We define:
\begin{equation*}
A_K = \sum_{k=0}^{K-1}\beta_k\partial_{\omega\lambda} G(\lambda, \omega_{k})\prod_{s={k+1}}^{K-1}\left(I - \beta_s\partial_{\omega^2} G(\lambda,\omega_s)\right)
\end{equation*}
and $A^* = \partial_{\omega\lambda}G(\lambda, \omega_{\lambda}) \partial_{\omega^2} G(\lambda, \omega_\lambda)^{-1}$. Then a recursive equation satisfied by $A_k$ is:
\begin{equation}
\label{eq:recursive_appendix}
\begin{split}
A_{k+1} =& A_k\left(I - \beta_k\partial_{\omega^2} G(\lambda,\omega_k)\right) + \beta_k\partial_{\omega\lambda} (\lambda, \omega_{k})\\
=& (1 -\beta_k)A_k + \beta_k(A_k( I  -\partial_{\omega^2} G(\lambda,\omega_k)) + \partial_{\omega\lambda} G(\lambda, \omega_{k}))\\
\end{split}
\end{equation}
then by reorganizing the terms in the definition of $A^*$, we have:
\begin{equation}
\label{eq:exact_appendix}
A^{*} = A^{*}(I - \partial_{\omega^2} G(\lambda,\omega_{\lambda})) + \partial_{\omega\lambda} G(\lambda, \omega_{\lambda})
\end{equation}
By combining Eq.~(\ref{eq:recursive_appendix}) and Eq.~(\ref{eq:exact_appendix}), we have:
\begin{equation*}
\begin{split}
&||A_{k+1} - A^{*}||\\
\le&(1- \beta_k)||A_k - A^*|| + \beta_k||A_k( I  -\partial_{\omega^2} G(\lambda,\omega_k)) + \partial_{\omega\lambda} G(\lambda, \omega_{k})\\
&- A^{*}(I - \partial_{\omega^2} G(\lambda,\omega_{\lambda})) - \partial_{\omega\lambda} G(\lambda, \omega_{\lambda})||\\
\overset{(i)}{\le}& (1- \beta_k)||A_k - A^*|| + \beta_k||A_k( I  -\partial_{\omega^2} G(\lambda,\omega_k)) -A^{*}(I - \partial_{\omega^2} G(\lambda,\omega_{\lambda}))||\\ 
&+ \beta_k||\partial_{\omega\lambda} G(\lambda, \omega_{k}) - \partial_{\omega\lambda} G(\lambda, \omega_{\lambda})||\\
\overset{(ii)}{\le}& (1- \beta_k)||A_k - A^*|| + \beta_k||( I  -\partial_{\omega^2} G(\lambda,\omega_k))||||A_k - A^*||\\
&+ \beta_k||A^*||||\partial_{\omega^2} G(\lambda, \omega_k) - \partial_{\omega^2} G(\lambda, \omega_{\lambda})|| + \beta_k||\partial_{\omega\lambda} G(\lambda, \omega_{k}) - \partial_{\omega\lambda} G(\lambda, \omega_{\lambda})||\\
\overset{(iii)}{\le}& (1- \beta_k + \beta_k||( I  -\partial_{\omega^2} G(\lambda,\omega_k)) ||)||A_k - A^*|| + \beta_k(L_{G, \omega\lambda}+ C_{G,\omega\lambda}L_{G,\omega^2}/\mu_G)e_{\omega,k}\\
\overset{(iv)}{\le}& (1- \mu_G\beta_k)||A_k - A^*|| + C_1\beta_ke_{\omega,k}\\
\end{split}
\end{equation*}
where (i) uses the triangle inequality; (ii) uses the Cauchy-schwarz inequality, (iii) uses the property $||A^*|| \le C_{G,\omega\lambda}/\mu_G$ (The property is straightforward, please refer to Lemma 2.2 in~\cite{grazzi2020iteration} for a proof); in (iv) we denote $C_1 = L_{G, \omega\lambda}+ C_{G,\omega\lambda}L_{G,\omega^2}/\mu_G$ for ease of writing.  In summary:
\begin{equation}
\label{eq:recur}
||A_{k+1} - A^{*}|| \le (1- \mu_G\beta_k)||A_k - A^*|| + C_1\beta_ke_{\omega,k}
\end{equation}
Multiplying both sides by $1/m_{k}$ and sum, we have:
\begin{equation*}
\begin{split}
||A_{K} - A^{*}|| 
&\le m_{K}\left(||A_0 - A^*|| + C_1\sum_{k=0}^{K-1}\beta_ke_{\omega,k}/m_k\right)\\
\end{split}
\end{equation*}
Then since $\nabla_{\lambda} f_K = \partial_{\lambda} F(\lambda, \omega_K) - A_Kp_K$
and $\nabla_{\lambda} f = \partial_{\lambda} F(\lambda, \omega_{\lambda}) - A^*\partial_{\omega} F(\lambda, \omega_{\lambda})$, so we have:
\begin{equation}
\label{eq:A_conv}
\begin{split}
||\nabla_{\lambda} f_K - \nabla_{\lambda} f|| =& ||\partial_{\lambda} F(\lambda, \omega_K) - A_Kp_K - \partial_{\lambda} F(\lambda, \omega_{\lambda}) + A^*\partial_{\omega} F(\lambda, \omega_{\lambda})||\\
\overset{(i)}{\le}& ||\partial_{\lambda} F(\lambda, \omega_K) - \partial_{\lambda} F(\lambda, \omega_{\lambda})|| + ||\partial_{\omega} F(\lambda, \omega_{\lambda})||||A_K -A^*|| + ||A^*||e_{p,K}\\
\overset{(ii)}{\le}& L_{F,\lambda}e_{\omega,K} + \frac{C_{G,\omega\lambda}}{\mu_G}e_{p,K} + L_{F,\omega}||A_K -A^*||\\
\end{split}
\end{equation}
where $(i)$ uses triangle inequality and Cauchy-Schwartz inequality; $(ii)$ uses Assumption~\ref{upper_assumption}.a. It is straightforward to verify $||\nabla_{\lambda} f_K - \nabla_{\lambda} f||\to 0$ as $k \to \infty$. if the conditions in the Lemma is satisfied.

\noindent Next for the forward mode, we take similar approaches. Suppose we define $v_K = \sum_{k=0}^{K-1}\beta_k\prod_{s={k+1}}^{K-1}\left(I -\beta_s\partial_{\omega^2} G(\lambda,\omega_s)\right)p_k$
and $v^* = \partial_{\omega^2} G(\lambda, \omega_\lambda)^{-1}\partial_{\omega}F(\lambda, \omega_{\lambda})$. Then we get the recursive equation for $v_k$ as:
\begin{equation}
\label{eq:recursive_appendix2}
\begin{split}
v_{k+1} =& \left(I - \beta_k\partial_{\omega^2} G(\lambda,\omega_k)\right)v_k + \beta_k p_k = (1 -\beta_k)v_k + \beta_k(( I  -\partial_{\omega^2} G(\lambda,\omega_k))v_k + p_k)\\
\end{split}
\end{equation}
then by reorganizing the terms in the definition of $v^*$, we have:
\begin{equation}
\label{eq:exact_appendix2}
v^{*} = (I - \partial_{\omega^2} G(\lambda,\omega_{\lambda}))v^{*} + \partial_{\omega}F(\lambda, \omega_{\lambda})
\end{equation}
By combine Eq.~(\ref{eq:recursive_appendix2}) and Eq.~(\ref{eq:exact_appendix2}), we have:
\begin{equation*}
\begin{split}
||v_{k+1} - v^{*}|| \le&(1- \beta_k)||v_k - v^*|| + \beta_k||( I  -\partial_{\omega^2} G(\lambda,\omega_k))v_k + p_k - (I - \partial_{\omega^2} G(\lambda,\omega_{\lambda}))v^{*} - \partial_{\omega}F(\lambda, \omega_{\lambda})||\\
\overset{(i)}{\le}& (1- \beta_k)||v_k - v^*|| + \beta_k||( I  -\partial_{\omega^2} G(\lambda,\omega_k))v_k - (I - \partial_{\omega^2} G(\lambda,\omega_{\lambda}))v^{*}|| + \beta_ke_{p,k}\\
\overset{(ii)}{\le}& (1- \beta_k)||v_k - v^*|| + \beta_k||( I  -\partial_{\omega^2} G(\lambda,\omega_k))||||v_k - v^*||\\
&+ \beta_k||v^*||||\partial_{\omega^2} G(\lambda, \omega_k) - \partial_{\omega^2} G(\lambda, \omega_{\lambda})|| + \beta_ke_{p,k}\\
\overset{(iii)}{\le}& (1- \beta_k + \beta_k||( I  -\partial_{\omega^2} G(\lambda,\omega_k)) ||)||v_k - v^*|| + \beta_k C_{F,\lambda}L_{G,\omega^2}/\mu_Ge_{\omega,k} + \beta_ke_{p,k}\\
\overset{(iv)}{\le}& (1- \mu_G\beta_k)||v_k - v^*|| + C_2\beta_ke_{\omega,k} + \beta_ke_{p,k}\\
\end{split}
\end{equation*}
where (i) uses the triangle inequality; (ii) uses the Cauchy-schwarz inequality; (iii) uses the property $||v^*|| \le C_{F,\lambda}/\mu_G$; (iv) denotes $C_2 = C_{F,\lambda}L_{G,\omega^2}/\mu_G$. In summary:
\begin{equation}
\label{eq:recur2}
||v_{k+1} - v^{*}|| \le (1- \mu_G\beta_k)||v_k - v^*|| + C_2\beta_ke_{\omega,k} + \beta_ke_{p,k}
\end{equation}
Multiplying both sides by $1/m_k$ and sum, then we have:
\begin{equation*}
\begin{split}
||v_{K} - v^{*}|| &\le m_{K}\left(||v_0 - v^*|| + C_2\sum_{k=0}^{K-1}\beta_ke_{\omega,k}/m_k +  \sum_{k=0}^{K-1}\beta_ke_{p,k}/m_k\right)
\end{split}
\end{equation*}
Finally, since $\nabla_{\lambda} f_K = \partial_{\lambda} F(\lambda, \omega_K) - \partial_{\omega\lambda} G(\lambda, \omega_K)v_K$
and $\nabla_{\lambda} f = \partial_{\lambda} F(\lambda, \omega_{\lambda}) - \partial_{\omega\lambda} G(\lambda, \omega_{\lambda})v^*$, so we have:
\begin{equation*}
\begin{split}
||\nabla_{\lambda} f_K - \nabla_{\lambda} f|| =& ||\partial_{\lambda} F(\lambda, \omega_K) + \partial_{\omega\lambda} G(\lambda, \omega_K)v_K - \partial_{\lambda} F(\lambda, \omega_{\lambda})) - \partial_{\omega\lambda} G(\lambda, \omega_{\lambda})v^*||\\
\overset{(i)}{\le}& ||\partial_{\lambda} F(\lambda, \omega_K) - \partial_{\lambda} F(\lambda, \omega_{\lambda})|| + ||\partial_{\omega\lambda} G(\lambda, \omega_{\lambda})|| *||v_K -v^*||\\
&+ ||v^*||*||\partial_{\omega\lambda} G(\lambda, \omega_{K}) - \partial_{\omega\lambda} G(\lambda, \omega_{\lambda})||\\
\overset{(ii)}{\le}& \left(L_{F,\lambda}+ \frac{C_{F,\lambda}L_{G, \omega\lambda}}{\mu_G}\right)e_{\omega,k} + C_{G,\omega\lambda}||v_K -v^*||\\
\end{split}
\end{equation*}
where $(i)$ uses triangle inequality and Cauchy-Schwartz inequality and $(ii)$ uses assumption~\ref{upper_assumption}.a. It is straightforward to verify $||\nabla_{\lambda} f_K - \nabla_{\lambda} f||\to 0$ as $k \to \infty$ if the conditions in the lemma are satisfied. This completes the proof.
\end{proof}

\noindent Next we show some corollaries that use specific sequences:
\begin{corollary}
\label{corollary:choice1}
Given a positive integer $K$, inner state $\hat{\omega}_K$ and constant $\beta$. We use backward mode of $s_k$ and set $\omega_{k} = \hat{\omega}_K$, $p_k = \partial_{\omega} F(\lambda, \hat{\omega}_K)$ and $\beta_k = \beta$ for $k \in [K]$. If $\exists\ \epsilon \in (0,1)$, such that $\beta_k \in (\epsilon/\mu_G, 1/\mu_G)$, then we have:
\begin{equation*}
||\nabla_{\lambda} f_K - \nabla_{\lambda} f|| = O(e_{\omega,K})
\end{equation*}
\end{corollary}

\begin{corollary}
\label{corollary:choice2}
Given a positive integer $K$, we have sequence $\{\hat{\omega}_k\}$ and $\{\hat{\beta}_k\}$ for $k\in[K]$. We use backward mode of $s_k$ and pick $\omega_{k} = \hat{\omega}_k$, $p_k = \partial_{\omega} F(\lambda, \hat{\omega}_K)$ and $\beta_k = \hat{\beta}_k$ for $k \in [K]$. Suppose we have $\beta_k = O(k^{-1})$ and $e_{\omega,k} =O(k^{-0.5})$, then:
\begin{equation*}
||\nabla_{\lambda} f_K - \nabla_{\lambda} f|| = O(K^{-{0.5}})
\end{equation*}
\end{corollary}

\begin{corollary}
\label{corollary:choice3}
Given a positive integer $K$, we have sequence $\{\hat{\omega}_k\}$ and $\{\hat{\beta}_k\}$ for $k\in[K]$. We use forward mode of $s_k$ and pick $\omega_{k} = \hat{\omega}_k$, $p_k = \partial_{\omega} F(\lambda, \hat{\omega}_k)$ and $\beta_k = \hat{\beta}_k$ for $k \in [K]$. Suppose we have $\beta_k = O(k^{-1})$ and $e_{\omega,k} =O(k^{-0.5})$, then:
\begin{equation*}
||\nabla_{\lambda} f_K - \nabla_{\lambda} f|| = O(K^{-{0.5}})
\end{equation*}
\end{corollary}

\begin{proof}
We first prove Corollary~\ref{corollary:choice1}. Based on Lemma~\ref{lemma:conv_cond}, Take corresponding values from Corollary~\ref{corollary:choice1}, we have:
\begin{equation*}
\begin{split}
||A_{K} - A^{*}|| &\le m_{K}\left(||A_0 - A^*|| + C_1\sum_{k=0}^{K-1}\beta_ke_{\omega,k}/m_k\right) \le (1-\epsilon)^{K+1}||A_0 - A^*|| + \frac{C_1e_{\omega,K}}{\mu_G}\sum_{k=0}^{K-1}(1-\epsilon)^{K-k}\\
&\le (1-\epsilon)^{K+1}||A_0 - A^*|| + \frac{C_1(1-\epsilon)(1 -(1-\epsilon)^K)}{\mu_G\epsilon}e_{\omega,K}\\
&\le (1-\epsilon)^{K+1}||A_0 - A^*|| + \frac{C_1(1-\epsilon)}{\mu_G\epsilon}e_{\omega,K} = O(e_{\omega,K})
\end{split}
\end{equation*}
The last equality holds if $e_{\omega,k}$ converges slower than linear rate. Then we combine with Eq.~(\ref{eq:A_conv}), and notice that $e_{p,K} = ||\partial_{\omega} F(\lambda, \hat{\omega}_K) - \partial_{\omega} F(\lambda, \omega_{\lambda})|| \le L_{F,\omega}e_{\omega, K}$. It is straightforward to verify $||\nabla_{\lambda} f_K - \nabla_{\lambda} f|| = O(e_{\omega,K})$. This completes the proof for Corollary~\ref{corollary:choice1}.

Next we use induction to prove Corollary~\ref{corollary:choice2}. By the condition in the corollary, we have $\beta_k = \frac{\theta_1}{k+1}$ and $e_{\omega,k} = \frac{\theta_2}{(k+1)^{0.5}}$ for some constants $\theta_1 > 1/2\mu_G$ and $\theta_2$. Now we pick $\theta_3 = max(||A_0 - A^*||, \frac{C_1\theta_1\theta_2}{\mu_G\theta_1 - 0.5})$ with $\theta_3$ some constant. Then the base of induction satisfies naturally, next suppose we have $||A_k - A^*||  \le \frac{\theta_3}{(k+1)^{0.5}}$. By Eq.~(\ref{eq:recur}):
\begin{equation*}
    \begin{split}
        ||A_{k+1} - A^{*}|| \le& (1- \frac{\mu_G\theta_1}{k+1})\frac{\theta_3}{(k+1)^{0.5}} + \frac{C_1\theta_1\theta_2}{(k+1)^{1.5}}
        \le (1- \frac{\mu_G\theta_1}{k+1})\frac{\theta_3}{(k+1)^{0.5}} + \frac{C_1\theta_1\theta_2}{(k+1)^{1.5}}\\
        \le& (1- \frac{\mu_G\theta_1}{k+1})\frac{\theta_3}{(k+1)^{0.5}} + \frac{\theta_3(\mu_G\theta_1 - 0.5)}{(k+1)^{1.5}}
        \le \frac{\theta_3}{(k+1)^{0.5}} + \frac{-0.5\theta_3}{(k+1)^{1.5}}\\
        \le& \frac{\theta_3}{(k+1)^{0.5}}\left(1 - \frac{1}{2(k+1)} \right) \le \frac{\theta_3}{(k+1)^{0.5}}\times \frac{(k+1)^{0.5}}{(k+2)^{0.5}} \le \frac{\theta_3}{(k+2)^{0.5}}
    \end{split}
\end{equation*}
To get  the second to last inequality, it is equivalent to get 
$\left(1 - \frac{1}{2(k+1)} \right)^2 \le 1 -\frac{1}{k+2}$, which is straightforward to verify. Then we combine with Eq.~(\ref{eq:A_conv}), and notice that $e_{p,K} = ||\partial_{\omega} F(\lambda, \hat{\omega}_K) - \partial_{\omega} F(\lambda, \omega_{\lambda})|| \le L_{F,\omega}e_{\omega, K}$. It is straightforward to verify $||\nabla_{\lambda} f_K - \nabla_{\lambda} f|| = O(e_{\omega,K}) = O(K^{-0.5})$. This finishes the proof for Corollary~\ref{corollary:choice2}. As for Corollary~\ref{corollary:choice3}, we follow similar induction steps as in the proof for Corollary~\ref{corollary:choice2}, but we instead use the recursive relation in Eq.~(\ref{eq:recur2})
\end{proof}
\noindent It is also possible to consider other sequences, for example, when $e_{\omega,k}$ converges in a slower rate such as $O(K^{0.25})$, but these cases are not as realistic and inspiring as the three corollaries here.

\section{Proof for Theorem 5 and the Lemmas}
\label{Appendix:D}
\subsection{Preliminary Results and Notations}
We first restate some definitions: $\omega_{\lambda_k} = \underset{\omega}{\arg\min} \ G(\lambda_k, \omega)$; $v_{\lambda_k} = \partial_{\omega^2} G(\lambda_k, \omega_{\lambda_k})^{-1}\partial_{\omega} F(\lambda_k, \omega_{\lambda_k})$; $\nabla f_{k} = \partial_{\lambda} F(\lambda_k, \omega_{k}) - \partial_{\omega\lambda} G(\lambda_k, \omega_{k})v_{k}$, $\nabla f_{k} (\xi_k)$ denotes its stochastic estimate as in Algorithm 1. The expectation in the subsequent sections is in terms of the randomness of $\epsilon_k$ as defined in the Algorithm~1 of the main text.
Next, we show some properties for $\omega_{\lambda}$ and $v_{\lambda}$ in the following proposition:
\begin{proposition}
\label{prop:5}
With Assumption~\ref{upper_assumption} and~\ref{lower_assumption} hold , we have:
\begin{equation*}
\begin{split}
||\omega_{\lambda_1} - \omega_{\lambda_2}|| \le C_{\omega}||\lambda_1 - \lambda_2||,
||v_{\lambda_1} - v_{\lambda_2}|| \le C_{v}||\lambda_1 - \lambda_2||   
\end{split}
\end{equation*}
Where $C_{\omega} = \frac{C_{G,\omega\lambda}}{\mu_G}$ and $C_{v} = \left(\frac{C_{F,\omega}L_{G, \omega\omega}}{\mu_G^2} + \frac{L_{F,\omega}}{\mu_G}\right)(1 + C_{\omega})$. Furthermore, we have $||v_{\lambda}|| \le \frac{C_{F,\omega}}{\mu_G}$, and we denote it as $M = \frac{C_{F,\omega}}{\mu_G}$.
\end{proposition}

\begin{proof}
The Lipschitzness of $\omega_{\lambda}$ is a rephrase of Lemma 2.2, case b)~\cite{ghadimi2018approximation}. Please refer to the paper for the proof. Then based on the definition of $v_{\lambda} = \partial_{\omega^2} G(\lambda, \omega_{\lambda})^{-1}\partial_{\omega} F(\lambda, \omega_{\lambda})$, we get its bound based on Assumption~\ref{upper_assumption}.b and~\ref{lower_assumption}.a; as for its Lipschitzness, we have:
\begin{equation*}
\begin{split}
||v_{\lambda_1} - v_{\lambda_2}|| =& ||\partial_{\omega^2} G(\lambda_1, \omega_{\lambda_1})^{-1}\partial_{\omega} F(\lambda_1, \omega_{\lambda_1}) - \partial_{\omega^2} G(\lambda_2, \omega_{\lambda_2})^{-1}\partial_{\omega} F(\lambda_2, \omega_{\lambda_2})||\\
\le& ||\partial_{\omega} F(\lambda_1, \omega_{\lambda_1})||||\partial_{\omega^2} G(\lambda_1, \omega_{\lambda_1})^{-1} - \partial_{\omega^2} G(\lambda_2, \omega_{\lambda_2})^{-1}||\\ 
&+ ||\partial_{\omega^2} G(\lambda_2, \omega_{\lambda_2})^{-1}||||\partial_{\omega} F(\lambda_1, \omega_{\lambda_1}) - \partial_{\omega} F(\lambda_2, \omega_{\lambda_2})||\\
\le& \left(\frac{C_{F,\omega}L_{G,\omega\omega}}{\mu_G^2} + \frac{L_{F,\omega}}{\mu_G}\right)(||\lambda_1 - \lambda_2|| + ||\omega_{\lambda_1} - \omega_{\lambda_2}||)\\
\le& \left(\frac{C_{F,\omega}L_{G,\omega\omega}}{\mu_G^2} + \frac{L_{F,\omega}}{\mu_G}\right)(1 + C_{\omega})||\lambda_1 - \lambda_2||
\end{split}   
\end{equation*}
To bound the term $||\partial_{\omega^2} G(\lambda_1, \omega_{\lambda_1})^{-1} - \partial_{\omega^2} G(\lambda_2, \omega_{\lambda_2})^{-1}||$, we use the fact that $||H_2^{-1} - H_1^{-1}|| = ||H_1^{-1}(H_2 - H_1)H_2^{-1}|| \le ||H_1^{-1}||||H_2 - H_1||||H_2^{-1}||$. Then it is straightforward to get the bound by Assumption~\ref{upper_assumption} and~\ref{lower_assumption}.
This completes the proof.
\end{proof}

Next we show two useful bounds for $||\omega_k - \omega_{\lambda_k}||^2$ and $||v_{k} - v_{\lambda_{k}}||^2$:
\begin{proposition}
\label{prop:6}
With Assumption~\ref{upper_assumption} and~\ref{lower_assumption}  hold, we have:
\begin{equation*}
\begin{split}
E[||\omega_k - \omega_{\lambda_k}||^2] \le& (1 + \gamma_{\omega}\alpha_{k-1})E[||\omega_{k} - \omega_{\lambda_{k-1}}||^2] + \frac{C_{\omega}^2\alpha_{k-1}}{\gamma_{\omega}}(1 + \gamma_{\omega}\alpha_{k-1}) E[||d_{k-1}||^2]
\end{split}
\end{equation*}
where $\gamma_{\omega}$ is some constant.
\end{proposition}
\begin{proof}
Firstly, we have:
\begin{equation*}
\begin{split}
E[||\omega_k - \omega_{\lambda_k}||^2]
\overset{(i)}{\le}& (1 + \gamma_{\omega}\alpha_{k-1})E[||\omega_{k} - \omega_{\lambda_{k-1}}||^2] + (1 + \frac{1}{\gamma_{\omega}\alpha_{k-1}}) E[||\omega_{\lambda_{k-1}} - \omega_{\lambda_k}||^2]\\
\overset{(ii)}{\le}& (1 + \gamma_{\omega}\alpha_{k-1})E[||\omega_{k} - \omega_{\lambda_{k-1}}||^2] + (1 + \frac{1}{\gamma_{\omega}\alpha_{k-1}})C_{\omega}^2 E[||\lambda_{k-1} - \lambda_k||^2]\\
\le& (1 + \gamma_{\omega}\alpha_{k-1})E[||\omega_{k} - \omega_{\lambda_{k-1}}||^2] + (1 + \frac{1}{\gamma_{\omega}\alpha_{k-1}})C_{\omega}^2 \alpha_{k-1}^2E[||d_{k-1}||^2]\\
\le& (1 + \gamma_{\omega}\alpha_{k-1})E[||\omega_{k} - \omega_{\lambda_{k-1}}||^2] + \frac{C_{\omega}^2\alpha_{k-1}}{\gamma_{\omega}}(1 + \gamma_{\omega}\alpha_{k-1}) E[||d_{k-1}||^2]\\
\end{split}
\end{equation*}
where $(i)$ uses Proposition A.1 and we take $\alpha$ as $\gamma_{\omega}\alpha_{k-1}$. $(ii)$ is based on Proposition C.1. This completes the proof.
\end{proof}

\begin{proposition}
\label{prop:7}
With Assumption~\ref{upper_assumption} and~\ref{lower_assumption} hold, we have:
\begin{equation*}
\begin{split}
E[||v_{k} - v_{\lambda_{k}}||^2] \le& (1 + \gamma_v\alpha_{k-1})E[||v_{k} - v_{\lambda_{k-1}}||^2] + \frac{C_{v}^2 \alpha_{k-1}}{\gamma_v}(1 + \gamma_v\alpha_{k-1})E[||d_{k-1}||^2]
\end{split}
\end{equation*}
where $\gamma_v$ is some constant.
\end{proposition}

\begin{proof}
Firstly, we have:
\begin{equation}
\begin{split}
E[||v_{k} - v_{\lambda_{k}}||^2]
\overset{(i)}{\le}& (1+\gamma_v\alpha_{k-1})E[||v_{k} - v_{\lambda_{k-1}}||^2] + (1 + \frac{1}{\gamma_v\alpha_{k-1}})E[||v_{\lambda_{k-1}} - v_{\lambda_{k}}||^2]\\
\overset{(ii)}{\le}& (1 +\gamma_v\alpha_{k-1})E[||v_{k} - v_{\lambda_{k-1}}||^2] + (1 + \frac{1}{\gamma_v\alpha_{k-1}})C_{v}^2E[||\lambda_{k-1} - \lambda_{k}||^2]\\
\le& (1 +\gamma_v\alpha_{k-1})E[||v_{k} - v_{\lambda_{k-1}}||^2] + (1 + \frac{1}{\gamma_v\alpha_{k-1}})C_{v}^2\alpha_{k-1}^2E[||d_{k-1}||^2]\\
\le& (1 + \gamma_v\alpha_{k-1})E[||v_{k} - v_{\lambda_{k-1}}||^2] + \frac{C_{v}^2 \alpha_{k-1}}{\gamma_v}(1 + \gamma_v\alpha_{k-1})E[||d_{k-1}||^2]\\
\end{split}
\end{equation}
where $(i)$ uses Proposition A.1 and we take $\alpha$ as $\gamma_{v}\alpha_{k-1}$. $(ii)$ is based on Proposition C.1. This completes the proof.
\end{proof}

\subsection{Proof for Iteration Progress Lemma 4}
\begin{lemma} (Lemma 4 in the main text)
With all the assumptions hold, we have:
\begin{equation*}
\begin{split}
E[f(\lambda_{k+1})] \le& f(\lambda_k) - \frac{\alpha_k}{2}||\nabla f(\lambda_k)||^2 + \alpha_k\Gamma_2^2E[||\omega_{k} - \omega_{\lambda_k}||^2] + 4\alpha_k\Gamma_1^2E[||v_{k} - v_{\lambda_{k}}||^2]\\
&+ \alpha_k E[||\nabla f_{k} -  d_k||^2] - \frac{\alpha_k}{2}(1 - \alpha_k L_f) E[||d_k||^2]\\
\end{split}
\end{equation*}
where $\Gamma_1^2 = C_{G,\omega\lambda}^2$, $\Gamma_2^2= 2L_{F,\lambda}^2 + 4L_{G,\omega\lambda}^2M^2$, and $M$ is denoted in Proposition~\ref{prop:5}.
\label{lemma:13}
\end{lemma}

\begin{proof}
Since $f(\lambda)$ is smooth, we have: $f(\lambda_{k+1}) \le f(\lambda_k) + \langle \nabla f(\lambda_k), \lambda_{k+1} - \lambda_k \rangle + \frac{L_f}{2} ||\lambda_{k+1} - \lambda_k||^2$ (The proof of $f(\lambda)$ is smooth can be found in Lemma 2.2 of \cite{ghadimi2018approximation}). Combine with the update rule $\lambda_{k+1} = \lambda_k -  \alpha_k d_k$, we have:
\begin{equation}
\label{eq:proof1}
\begin{split}
E[f(\lambda_{k+1})] \le& f(\lambda_k) -  \alpha_k E[\langle \nabla f(\lambda_k), d_k \rangle] + \frac{L_f\alpha_k^2}{2}  E[||d_k||^2] \\
\le& f(\lambda_k) - \frac{\alpha_k}{2}||\nabla f(\lambda_k)||^2 - \frac{\alpha_k}{2}E[||d_k||^2] + \frac{\alpha_k}{2}E[||\nabla f(\lambda_k) - d_k||^2] + \frac{L_f\alpha_k^2}{2} E[||d_k||^2]\\
\le& f(\lambda_k) - \frac{\alpha_k}{2}||\nabla f(\lambda_k)||^2 + \frac{\alpha_k}{2}E[||\nabla f(\lambda_k) - d_k||^2] - \frac{\alpha_k}{2}(1 - \alpha_k L_f)  E[||d_k||^2]\\
\le& f(\lambda_k) - \frac{\alpha_k}{2}||\nabla f(\lambda_k)||^2 + \alpha_k \underbrace{E[||\nabla f(\lambda_k) - \nabla f_{k}||^2}_{\Pi} + \alpha_k E[||\nabla f_{k} -  d_k||^2]\\
&-\frac{\alpha_k}{2}(1 - \alpha_k L_f) E[||d_k||^2]\\
\end{split}
\end{equation}
In the last inequality, we use the relaxed triangle inequality. Next we bound the term $\Pi$. By the definition of $\nabla f_{k}$ and $\nabla f(\lambda_k)$, we have:
\begin{equation}
\label{eq:proof2}
\begin{split}
&E[||\nabla f_{k} - \nabla f(\lambda_k)||^2]\\
=&E[||\partial_{\lambda} F(\lambda_k, \omega_{k}) - \partial_{\lambda} F(\lambda_k, \omega_{\lambda_k}) - \left(\partial_{\omega\lambda} G(\lambda_k, \omega_{k})  v_{k} - \partial_{\omega\lambda} G(\lambda_k, \omega_{\lambda_{k}})  v_{\lambda_{k}}\right)||^2]\\
\overset{(i)}{\le}&2E[||\partial_{\lambda} F(\lambda_k, \omega_{k}) - \partial_{\lambda} F(\lambda_k, \omega_{\lambda_k})||^2] + 2E[||\partial_{\omega\lambda} G(\lambda_k, \omega_{k})  v_{k} - \partial_{\omega\lambda} G(\lambda_k, \omega_{\lambda_{k}})  v_{\lambda_{k}}||^2]\\
\overset{(ii)}{\le}&2L_{F,\lambda}^2E[||\omega_{k} - \omega_{\lambda_k}||^2] +  4E[||\partial_{\omega\lambda} G(\lambda_k, \omega_{k})(v_{k} - v_{\lambda_{k}})||^2]\\
&+ 4E[||(\partial_{\omega\lambda} G(\lambda_k, \omega_{k}) - \partial_{\omega\lambda} G(\lambda_k, \omega_{\lambda_{k}}))v_{\lambda_{k}}||^2]\\
\overset{(iii)}{\le}&2L_{F,\lambda}^2E[||\omega_{k} - \omega_{\lambda_k}||^2] +  4E[||\partial_{\omega\lambda} G(\lambda_k, \omega_{k})||^2*||v_{k} - v_{\lambda_{k}}||^2]\\
&+ 4E[||\partial_{\omega\lambda} G(\lambda_k, \omega_{k}) - \partial_{\omega\lambda} G(\lambda_k, \omega_{\lambda_{k}}) ||^2*||v_{\lambda_{k}}||^2]\\
\overset{(iv)}{\le}& (2L_{F,\lambda}^2 + 4M^2L_{G,\omega\lambda}^2)E[||\omega_{k} - \omega_{\lambda_k}||^2] + 4C_{G,\omega\lambda}^2E[||v_{k} - v_{\lambda_{k}}||^2]\\
\le& \Gamma_2^2E[||\omega_{k} - \omega_{\lambda_k}||^2] + 4\Gamma_1^2E[||v_{k} - v_{\lambda_{k}}||^2]
\end{split}
\end{equation}
$(i)$ and $(ii)$ uses the relaxed triangle inequality, $(iii)$ is by the smoothness Assumption~\ref{upper_assumption}.a for the first term and the Cauchy-Schwartz inequality for the second and the third term; $(iv)$ uses the Assumption~\ref{lower_assumption}.b,~\ref{lower_assumption}.c and Proposition~\ref{prop:5}. Combine Eq.~(\ref{eq:proof1}) and Eq.~(\ref{eq:proof2}), we have:
\begin{equation*}
\begin{split}
E[f(\lambda_{k+1})] \le& f(\lambda_k) - \frac{\alpha_k}{2}||\nabla f(\lambda_k)||^2 + \alpha_k\Gamma_2^2E[||\omega_{k} - \omega_{\lambda_k}||^2]\\
&+ 4\alpha_k\Gamma_1^2E[||v_{k} - v_{\lambda_{k}}||^2] + \alpha_k E[||\nabla f_{k} -  d_k||^2] - \frac{\alpha_k}{2}(1 - \alpha_k L_f) E[||d_k||^2]\\
\end{split}
\end{equation*}
This completes the proof.
\end{proof}

\subsection{Bounds for the Three Types of Errors}
As shown by Lemma~\ref{lemma:13} of Section D.2, we have three types of errors: $||\omega_{k} - \omega_{\lambda_k}||^2$, $||v_{k} - v_{\lambda_{k}}||^2$ and $||\nabla f_{k} -  d_k||^2$. We bound these errors in the subsequent three lemmas.

\begin{lemma}
With Assumption~\ref{upper_assumption},~\ref{lower_assumption} and~\ref{noise_assumption_append} hold, and for $0 < \eta_k < 1$, we have:
\begin{equation*}
\begin{split}
E[||d_k - \nabla f_{k}||^2] \le& (1 - \eta_k)E[||d_{k-1} - \nabla f_{k-1}||^2]
+  C_{d,d}\alpha_{k-1}^2E[||d_{k-1}||^2] + C_{d,\omega}\alpha_{k-1}^2E[||\omega_{k-1} - \omega_{\lambda_{k-1}}||^2]\\ &+ C_{d,v}\alpha_{k-1}^2E[||v_{k-1} - v_{\lambda_{k-1}}||^2] + C_{d,n}\alpha_{k-1}^2\sigma^2
\end{split}
\end{equation*}
where $C_{d,d} = 2\Gamma_2^2, C_{d,\omega} = 2(c_{\tau}^2\Gamma_2^2\Gamma_3^2 + 4c_{\beta}^2\Gamma_1^2\Gamma_4^2), C_{d,v} = 32c_{\beta}^2\Gamma_1^2\Gamma_3^2, C_{d,n} = 2(c_{\tau}^2\Gamma_2^2 + 4c_{\beta}^2(1 + M^2)\Gamma_1^2 + c_{\eta}^2(1+M^2))$, $c_{\beta}$, $c_{\tau}$ and $c_{\eta}$ are some constants.
\label{lemma:14}
\end{lemma}
\begin{proof}
By the definition of $d_k$, we have:
\begin{equation}
\label{eq:proof3}
\begin{split}
&E[||d_k - \nabla f_{k}||^2] = E[||\nabla f_{k} (\xi_k) + (1 - \eta_k) (d_{k-1} - \nabla f_{k-1} (\xi_k)) - \nabla f_{k}||^2]\\
=&E[||(1 - \eta_k) (d_{k-1} - \nabla f_{k-1}) + \nabla f_{k} (\xi_k) - \nabla f_{k} + (1 - \eta_k)* (\nabla f_{k-1} - \nabla f_{k-1} (\xi_k))||^2]\\
\overset{(i)}{=}&(1 - \eta_k)^2E[||d_{k-1} - \nabla f_{k-1}||^2] + E[||\eta_k(\nabla f_{k} (\xi_k) - \nabla f_{k})\\
&+ (1 - \eta_k)(\nabla f_{k}(\xi_k) - \nabla f_{k-1} (\xi_k)-  (\nabla f_k - \nabla f_{k-1}))||^2]\\
\overset{(ii)}{\le}&(1 - \eta_k)^2E[||d_{k-1} - \nabla f_{k-1}||^2] + 2\eta_k^2E[||\nabla f_{k} (\xi_k) - \nabla f_{k}||^2]\\
&+ 2(1 - \eta_k)^2 E[||\nabla f_{k}(\xi_k) - \nabla f_{k-1} (\xi_k)-  (\nabla f_k - \nabla f_{k-1})||^2]\\
\overset{(iii)}{\le}& (1 - \eta_k)^2E[||d_{k-1} - \nabla f_{k-1}||^2] + 2(1+M^2)\eta_k^2\sigma^2 + 2(1 - \eta_k)^2 \underbrace{E[||\nabla f_{k}(\xi_k) - \nabla f_{k-1} (\xi_k)||^2]}_{\Pi}\\
\end{split}
\end{equation}
$(i)$ is because $\xi_k$ is independent of $d_{k-1}$, $(ii)$ is because of the relaxed triangle inequality, and $(iii)$ uses the inequality $E[||X||^2] \ge E[||X - E[X]||^2]$. For the term $\Pi$:
\begin{equation}
\label{eq:eq9}
\begin{split}
&E[||\nabla f_{k}(\xi_k) - \nabla f_{k-1} (\xi_k)||^2]\\
=&  E[||\partial_{\lambda} F(\lambda_k, \omega_k; \xi_{k,4}) - \partial_{\omega\lambda} G(\lambda_k, \omega_k; \xi_{k,5})v_k - \partial_{\lambda} F(\lambda_{k-1}, \omega_{k-1}; \xi_{k,4}) + \partial_{\omega\lambda} G(\lambda_{k-1}, \omega_{k-1}; \xi_{k,5})v_{k-1}||^2]\\
\overset{(i)}{\le}& 2E[||\partial_{\lambda} F(\lambda_k, \omega_{k}; \xi_{k,4})- \partial_{\lambda} F(\lambda_{k-1}, \omega_{k-1}; \xi_{k,4})||^2] + 2E[||\partial_{\omega\lambda} G(\lambda_k, \omega_{k}; \xi_{k,5})v_{k} - \partial_{\omega\lambda} G(\lambda_{k-1}, \omega_{k-1}; \xi_{k,5})v_{k-1}||^2]\\
\overset{(ii)}{\le}& 2E[||\partial_{\lambda} F(\lambda_k, \omega_{k}; \xi_{k,4})- \partial_{\lambda} F(\lambda_{k-1}, \omega_{k-1}; \xi_{k,4})||^2]\\
&+ 4M^2E[||\partial_{\omega\lambda} G(\lambda_k, \omega_{k}; \xi_{k,5}) - \partial_{\omega\lambda} G(\lambda_{k-1}, \omega_{k-1}; \xi_{k,5})||^2] + 4C_{G,\omega\lambda}^2E[||v_{k} - v_{k-1}||^2]\\
\le& (2L_{F, \lambda}^2 + 4M^2L_{G,\omega\lambda}^2)(E[||\lambda_k - \lambda_{k-1}||^2] + E[||\omega_{k} - \omega_{k-1}||^2]) + 4C_{G,\omega\lambda}^2E[||v_{k} - v_{k-1}||^2]\\
\le& \alpha_{k-1}^2\Gamma_2^2E[||d_{k-1}||^2] + \underbrace{\Gamma_2^2E[||\omega_{k} - \omega_{k-1}||^2]}_{\Pi_1} + 4\Gamma_1^2\underbrace{E[||v_{k} - v_{k-1}||^2]}_{\Pi_2}\\
\end{split}
\end{equation}
$(i)$ is by the relaxed triangle inequality, $(ii)$ uses Cauchy-Schwartz inequality, Assumption~\ref{lower_assumption}.b and~\ref{lower_assumption}.c, we also use $||v_k|| \le M$, this can be proved by induction. Assume $||v_0||\le M$, then by the triangle inequality and Cauchy-Schwartz inequality we have: $||v_{k+1}|| \le \beta_{k+1}||\partial_{\omega} F(\lambda_{k+1}, \omega_{k}; \xi_{k+1,2})|| + ||(I - \beta_{k+1}\partial_{\omega^2} G(\lambda_{k+1},\omega_k; \xi_{k+1,3}))||||v_k|| \le \beta_{k+1}C_{F,\omega} + (1 -\beta_{k+1}\mu_G)M = M$.

\noindent Next we bound the term $\Pi_1$ and $\Pi_2$ separately. For the term $\Pi_1$, by separating mean and variance, we firstly have:
\[
E[||\omega_{k} - \omega_{k-1}||^2] = \tau_{k}^2 E[||\partial_{\omega}G(\lambda_k, \omega_{k-1}; \xi_{k,1})||^2] \le \tau_{k}^2E[||\partial_{\omega}G(\lambda_k, \omega_{k-1})||^2] + \tau_{k}^2\sigma^2\\
\]
As for  the term $E[||\partial_{\omega}G(\lambda_k, \omega_k; \xi_{k,1})||^2]$, we have:
\begin{equation*}
\begin{split}
E[||\partial_{\omega}G(\lambda_k, \omega_{k-1})||^2]
\overset{(i)}{=}& E[||\partial_{\omega}G(\lambda_k, \omega_{k-1}) - \partial_{\omega}G(\lambda_k, \omega_{\lambda_k})||^2] \overset{(ii)}{\le} L_{G,\omega}^2E[||\omega_{k-1} - \omega_{\lambda_k}||^2]\\
\end{split}
\end{equation*}
$(i)$ uses the definition of $\omega_{\lambda_k}$; $(ii)$ uses Assumption~\ref{lower_assumption}.b. Next for the term $\Pi_2$, firstly by separating mean and variance, we have:
\begin{equation*}
\begin{split}
E[||v_{k} - v_{k-1}||^2] =& \beta_{k}^2E[||\partial_{\omega} F(\lambda_k, \omega_{k-1}; \xi_{k,2}) - \partial_{\omega^2} G(\lambda_k, \omega_{k-1}; \xi_{k,3}) * v_{k-1}||^2]\\
\le& \beta_{k}^2E[||\partial_{\omega} F(\lambda_k, \omega_{k-1}) - \partial_{\omega^2} G(\lambda_k, \omega_{k-1}) * v_{k-1}||^2] + \beta_{k}^2(1 + M^2)\sigma^2
\end{split}
\end{equation*}
As for the term $E[||\partial_{\omega} F(\lambda_k, \omega_{k-1}) - \partial_{\omega^2} G(\lambda_k, \omega_{k-1}) * v_{k-1}||^2]$, we have:
\begin{equation*}
\begin{split}
&E[||\partial_{\omega} F(\lambda_k, \omega_{k-1}) - \partial_{\omega^2} G(\lambda_k, \omega_{k-1}) * v_{k-1}||^2] \\
\overset{(i)}{\le}& E[||\partial_{\omega} F(\lambda_k, \omega_{k-1}) - \partial_{\omega^2} G(\lambda_k, \omega_{k-1}) * v_{k-1} - (\partial_{\omega} F(\lambda_k, \omega_{\lambda_k}) - \partial_{\omega^2} G(\lambda_k, \omega_{\lambda_k}) * v_{\lambda_k})||^2] \\
\overset{(ii)}{\le}& 2E[||\partial_{\omega} F(\lambda_k, \omega_{k-1}) - \partial_{\omega} F(\lambda_k, \omega_{\lambda_k}) ||^2] +  2E[||\partial_{\omega^2} G(\lambda_k, \omega_{k-1}) * v_{k-1} - \partial_{\omega^2} G(\lambda_k, \omega_{\lambda_k}) * v_{\lambda_k}||^2]\\
&\\
\overset{(iii)}{\le}& 2L_{F,\omega}^2E[|| \omega_{k-1} - \omega_{\lambda_k}||^2] +  4M^2E[||\partial_{\omega^2} G(\lambda_k, \omega_{k-1}) - \partial_{\omega^2} G(\lambda_k, \omega_{\lambda_k})||^2] + 4L_{G,\omega}^2E[|v_{k-1} - v_{\lambda_k}|||^2]\\
&\\
\overset{(iv)}{\le}& 2(L_{F,\omega}^2 + 2M^2L_{G,\omega\omega}^2)E[|| \omega_{k-1} - \omega_{\lambda_k}||^2]  + 4L_{G,\omega}^2E[|v_{k-1} - v_{\lambda_k}|||^2] \\
\le& \Gamma_4^2E[|| \omega_{k-1} - \omega_{\lambda_k}||^2]  + 4\Gamma_3^2E[|v_{k-1} - v_{\lambda_k}|||^2] \\
\end{split}
\end{equation*}
where $(i)$ uses the definition of $v_{\lambda_k}$; $(ii)$ uses generalized triangle inequality; $(iii)$ uses Assumption~\ref{lower_assumption}.b; $(iv)$ uses the Assumption~\ref{lower_assumption}.c. Put everything back into Eq.~(\ref{eq:eq9}), and re-organize the terms, we get:
\begin{equation}
\label{eq:eq34}
\begin{split}
&E[||\nabla f_{k}(\xi_k) - \nabla f_{k-1} (\xi_k)||^2]\\
\le& \alpha_{k-1}^2\Gamma_2^2E[||d_{k-1}||^2] + (\tau_{k}^2\Gamma_2^2\Gamma_3^2 + 4\beta_{k}^2\Gamma_1^2\Gamma_4^2)E[||\omega_{k-1} - \omega_{\lambda_k}||^2] + 16\beta_{k}^2\Gamma_1^2\Gamma_3^2E[|v_{k-1} - v_{\lambda_k}|||^2]\\
&+ (\tau_{k}^2\Gamma_2^2 + 4\beta_{k}^2(1 + M^2)\Gamma_1^2)\sigma^2
\end{split}
\end{equation}
Finally, we put Eq.~(\ref{eq:eq34}) back into Eq.~(\ref{eq:proof3}) and get:
\begin{equation*}
\begin{split}
E[||d_k - \nabla f_{k}||^2]\le& (1 - \eta_k)^2E[||d_{k-1} - \nabla f_{k-1}||^2] + 2(1 - \eta_k)^2\alpha_{k-1}^2\Gamma_2^2E[||d_{k-1}||^2]\\
&+ 2(1 - \eta_k)^2(\tau_k^2\Gamma_2^2\Gamma_3^2 + 4\beta_k^2\Gamma_1^2\Gamma_4^2)E[||\omega_{k-1} - \omega_{\lambda_k}||^2] + 32(1 - \eta_k)^2\beta_k^2\Gamma_1^2\Gamma_3^2E[|v_{k-1} - v_{\lambda_k}|||^2]\\
&+ 2(1 - \eta_k)^2(\tau_k^2\Gamma_2^2 + 4\beta_k^2(1 + M^2)\Gamma_1^2)\sigma^2 + 2(1+M^2)\eta_k^2\sigma^2
\end{split}
\end{equation*}
Suppose we have $\tau_k = c_{\tau}\alpha_{k-1}$, $\beta_k=c_{\beta}\alpha_{k-1}$ and $\eta_k=c_{\eta}\alpha_{k-1}$, then:
\begin{equation*}
\begin{split}
&E[||d_k - \nabla f_{k}||^2]\\
\le& (1 - \eta_k)^2E[||d_{k-1} - \nabla f_{k-1}||^2] + 2(1 - \eta_k)^2\alpha_{k-1}^2\Gamma_2^2E[||d_{k-1}||^2]\\
&+ 2(1 - \eta_k)^2\alpha_{k-1}^2(c_{\tau}^2\Gamma_2^2\Gamma_3^2 + 4c_{\beta}^2\Gamma_1^2\Gamma_4^2)E[||\omega_{k-1} - \omega_{\lambda_k}||^2] + 32(1 - \eta_k)^2\alpha_{k-1}^2c_{\beta}^2\Gamma_1^2\Gamma_3^2E[|v_{k-1} - v_{\lambda_k}|||^2]\\
&+ 2(1 - \eta_k)^2\alpha_{k-1}^2(c_{\tau}^2\Gamma_2^2 + 4c_{\beta}^2(1 + M^2)\Gamma_1^2)\sigma^2 + 2(1+M^2)c_{\eta}^2\alpha_{k-1}^2\sigma^2\\
\le& (1 - \eta_k)E[||d_{k-1} - \nabla f_{k-1}||^2] + 2(1 - \eta_k)^2\alpha_{k-1}^2\Gamma_2^2E[||d_{k-1}||^2]\\
&+ 2\alpha_{k-1}^2(c_{\tau}^2\Gamma_2^2\Gamma_3^2 + 4c_{\beta}^2\Gamma_1^2\Gamma_4^2)E[||\omega_{k-1} - \omega_{\lambda_k}||^2] + 32\alpha_{k-1}^2c_{\beta}^2\Gamma_1^2\Gamma_3^2E[|v_{k-1} - v_{\lambda_k}|||^2]\\
&+ 2\alpha_{k-1}^2(c_{\tau}^2\Gamma_2^2 + 4c_{\beta}^2(1 + M^2)\Gamma_1^2)\sigma^2 + 2(1+M^2)c_{\eta}^2\alpha_{k-1}^2\sigma^2\\
\end{split}
\end{equation*}
we use the fact that $(1 - \eta_k)^2 < (1 - \eta_k) < 1$ in the last inequality. Then combine with Proposition~\ref{prop:6} and Proposition~\ref{prop:7}, we get:
\begin{equation*}
\begin{split}
&E[||d_k - \nabla f_{k}||^2]\\
\le& (1 - \eta_k)E[||d_{k-1} - \nabla f_{k-1}||^2] + 2\alpha_{k-1}^2\Gamma_2^2E[||d_{k-1}||^2] + 2\alpha_{k-1}^2(c_{\tau}^2\Gamma_2^2\Gamma_3^2 + 4c_{\beta}^2\Gamma_1^2\Gamma_4^2)(1 + \gamma_{\omega}\alpha_{k-1})E[||\omega_{k-1} - \omega_{\lambda_{k-1}}||^2]\\
&+ 32\alpha_{k-1}^2c_{\beta}^2\Gamma_1^2\Gamma_3^2(1 + \gamma_v\alpha_{k-1})E[|v_{k-1} - v_{\lambda_{k-1}}|||^2]+ 2\alpha_{k-1}^2(c_{\tau}^2\Gamma_2^2 + 4c_{\beta}^2(1 + M^2)\Gamma_1^2 + c_{\eta}^2(1+ M^2))\sigma^2\\
\le& (1 - \eta_k)E[||d_{k-1} - \nabla f_{k-1}||^2] + 2\Gamma_2^2\alpha_{k-1}^2E[||d_{k-1}||^2]+ 2(c_{\tau}^2\Gamma_2^2\Gamma_3^2 + 4c_{\beta}^2\Gamma_1^2\Gamma_4^2)\alpha_{k-1}^2E[||\omega_{k-1} - \omega_{\lambda_{k-1}}||^2]\\
&+ 32c_{\beta}^2\Gamma_1^2\Gamma_3^2\alpha_{k-1}^2E[|v_{k-1} - v_{\lambda_{k-1}}|||^2] + 2\left(c_{\tau}^2\Gamma_2^2 + 4c_{\beta}^2(1 + M^2)\Gamma_1^2 + c_{\eta}^2(1+M^2)\right)\alpha_{k-1}^2\sigma^2\\
\end{split}
\end{equation*}
In the first inequality, we omit the higher order terms of $\alpha_{k-1}$ for $E[||d_{k-1}||^2]$; In the second inequality,  notice $1 + \gamma_{\omega}\alpha_{k-1}$ and $1 + \gamma_v\alpha_{k-1}$ is $O(1)$. Then we get the inequality shown in the lemma. This completes the proof.
\end{proof}

\begin{lemma}
With Assumption~\ref{upper_assumption},~\ref{lower_assumption},~\ref{noise_assumption_append} hold and $0 < \beta_k < 1/\mu_G$, we have:
\begin{equation*}
\begin{split}
E[||v_{k} - v_{\lambda_{k}}||^2] \le& (1 + \gamma_v\alpha_{k-1})^3(1 - \beta_k\mu_G)E[||v_{k-1} - v_{\lambda_{k-1}}||^2] \\
&+ C_{v,\omega}\alpha_{k-1}E[||\omega_{k-1} - \omega_{\lambda_{k-1}}||^2] + C_{v,d}\alpha_{k-1}E[||d_{k-1}||^2] + C_{v,n}\alpha_{k-1}^2\sigma^2\\
\end{split}
\end{equation*}
Where $C_{v,\omega} = c_{\beta}^2\left(L_{F,\omega}^2 + M^2L_{G,\omega\omega}^2 \right)/\gamma_v$, $C_{v,d} = C_v^2/\gamma_v$, $C_{v,n}= c_{\beta}^2(1 + M^2)$  ,$\gamma_v$, $c_{\beta}$ are some constants.
\label{lemma:15}
\end{lemma}

\begin{proof}
Following the definition of $v_{\lambda}$, it is easy to get: $v_{\lambda} = \beta\partial_{\omega} F(\lambda, \omega_{\lambda}) + \left(I - \beta\partial_{\omega^2} G(\lambda, \omega_{\lambda})\right) * v_{\lambda}$.
Next by the update rule of $v_{k}$ and separating mean and variance, we have:
\begin{equation*}
\begin{split}
E[||v_{k} - v_{\lambda_{k}}||^2]=&E[||\beta_k\partial_{\omega} F(\lambda_k, \omega_{k-1}; \xi_{k,2}) + \left(I - \beta_k\partial_{\omega^2} G(\lambda_k, \omega_{k-1}; \xi_{k,3})\right) * v_{k-1} \\
&- \left(\beta_k\partial_{\omega} F(\lambda_k, \omega_{\lambda_k}) + \left(I - \beta_k\partial_{\omega^2} G(\lambda_k, \omega_{\lambda_{k}})\right) * v_{\lambda_k}\right)||^2]\\
\le& E[||\beta_k\partial_{\omega} F(\lambda_k, \omega_{k-1}) + \left(I - \beta_k\partial_{\omega^2} G(\lambda_k, \omega_{k-1})\right) * v_{k-1} \\
&- \left(\beta_k\partial_{\omega} F(\lambda_k, \omega_{\lambda_k}) + \left(I - \beta_k\partial_{\omega^2} G(\lambda_k, \omega_{\lambda_{k}})\right) * v_{\lambda_k}\right)||^2] + (1 + M^2)\beta_k^2\sigma^2\\
\end{split}
\end{equation*}
For the first term in the above equation, we have:
\begin{equation*}
\begin{split}
&E[||\beta_k\partial_{\omega} F(\lambda_k, \omega_{k-1}) + \left(I - \beta_k\partial_{\omega^2} G(\lambda_k, \omega_{k-1})\right) * v_{k-1} - \left(\beta_k\partial_{\omega} F(\lambda_k, \omega_{\lambda_k}) + \left(I - \beta_k\partial_{\omega^2} G(\lambda_k, \omega_{\lambda_{k}})\right) * v_{\lambda_k}\right)||^2] \\
\overset{(i)}{\le}& (1 + \frac{1}{\gamma_v\alpha_{k-1}})\beta_k^2E[||\partial_{\omega} F(\lambda_k, \omega_{k-1}) - \partial_{\omega} F(\lambda_k, \omega_{\lambda_k})||^2]\\ 
&+ (1 + \gamma_v\alpha_{k-1})E[||(I - \beta_k\partial_{\omega^2} G(\lambda_k, \omega_{k-1})) * v_{k-1} - (I - \beta_k\partial_{\omega^2} G(\lambda_k, \omega_{\lambda_{k}})) * v_{\lambda_k}||^2] \\
\overset{(ii)}{\le}& (1 + \frac{1}{\gamma_v\alpha_{k-1}})L_{F,\omega}^2\beta_k^2E[||\omega_{k-1} - \omega_{\lambda_k}||^2] + (1 + \gamma_v\alpha_{k-1})(1 + \frac{1}{\gamma_v\alpha_{k-1}})M^2\beta_k^2E[||\partial_{\omega^2} G(\lambda_k, \omega_{k-1}) - \partial_{\omega^2} G(\lambda_k, \omega_{\lambda_{k}})||^2]\\
&+ (1 + \gamma_v\alpha_{k-1})^2E[||(I - \beta_k\partial_{\omega^2} G(\lambda_k, \omega_{k-1})) * (v_{k-1} - v_{\lambda_k})||^2] \\
\overset{(iii)}{\le}& (1 + \frac{1}{\gamma_v\alpha_{k-1}})L_{F,\omega}^2\beta_k^2E[||\omega_{k-1} - \omega_{\lambda_k}||^2] + (1 + \gamma_v\alpha_{k-1})(1 + \frac{1}{\gamma_v\alpha_{k-1}})M^2L_{G,\omega\omega}^2\beta_k^2E[||\omega_{k-1} - \omega_{\lambda_k}||^2]\\
&+ (1 + \gamma_v\alpha_{k-1})^2(1 - \beta_k\mu_G)^2E[||(v_{k-1} - v_{\lambda_k})||^2] \\
\le& \left((1 + \frac{1}{\gamma_v\alpha_{k-1}})L_{F,\omega}^2 + (1 + \gamma_v\alpha_{k-1})(1 + \frac{1}{\gamma_v\alpha_{k-1}})M^2L_{G,\omega\omega}^2 \right)\beta_k^2E[||\omega_{k-1} - \omega_{\lambda_k}||^2]\\
&+(1 + \gamma_v\alpha_{k-1})^2(1 - \beta_k\mu_G)^2E[||(v_{k-1} - v_{\lambda_k})||^2]
\end{split}
\end{equation*}
where $(i)$, $(ii)$ uses relaxed triangle inequality case 2; $(iii)$ uses Assumption~\ref{lower_assumption}. Finally, combine the above two equations together we have:
\begin{equation*}
\begin{split}
E[||v_{k} - v_{\lambda_{k}}||^2]\le& \left((1 + \frac{1}{\gamma_v\alpha_{k-1}})L_{F,\omega}^2 + (1 + \gamma_v\alpha_{k-1})(1 + \frac{1}{\gamma_v\alpha_{k-1}})M^2L_{G,\omega\omega}^2 \right)\beta_k^2E[||\omega_{k-1} - \omega_{\lambda_k}||^2]\\
&+(1 + \gamma_v\alpha_{k-1})^2(1 - \beta_k\mu_G)^2E[||(v_{k-1} - v_{\lambda_k})||^2] + (1 + M^2)\beta^2\sigma^2\\
\end{split}
\end{equation*}
Then we combine with Proposition~\ref{prop:6} and~\ref{prop:7} to have:
\begin{equation*}
\begin{split}
E[||v_{k} - v_{\lambda_{k}}||^2] \le& (1 + \gamma_v\alpha_{k-1})^3(1 - \beta_k\mu_G)^2E[||v_{k-1} - v_{\lambda_{k-1}}||^2] + \beta_k^2(1 + M^2)\sigma^2\\
&+ \left((1 + \frac{1}{\gamma_v\alpha_{k-1}})L_{F,\omega}^2 + (2 +
\gamma_v\alpha_{k-1} + \frac{1}{\gamma_v\alpha_{k-1}})M^2L_{G,\omega\omega}^2 \right)\beta_k^2(1 + \gamma_{\omega}\alpha_{k-1})E[||\omega_{k-1} - \omega_{\lambda_{k-1}}||^2]\\
&\overset{(ii)}{+}\left((1 + \frac{1}{\gamma_v\alpha_{k-1}})L_{F,\omega}^2 + (2 +
\gamma_v\alpha_{k-1} + \frac{1}{\gamma_v\alpha_{k-1}})M^2L_{G,\omega\omega}^2 \right)\beta_k^2\frac{C_{\omega}^2\gamma_v\alpha_{k-1}}{\gamma_{\omega}}(1 + \gamma_{\omega}\alpha_{k-1})E[||d_{k-1}||^2]\\
&\overset{(iii)}{+} C_v^2(1 + \gamma_v\alpha_{k-1})^2(1 + \frac{1}{\gamma_v\alpha_{k-1}})(1 - \beta_k\mu_G)^2 \gamma_v\alpha_{k-1}^2E[||d_{k-1}||^2]\\
\end{split}
\end{equation*}
Notice that $(1 + \gamma_{\omega}\alpha_{k-1})$ is O(1), so the term $(ii)$ is O($\alpha_k^2$), while $(iii)$ is O($\alpha_k$), so we can omit the term $(ii)$. Then use the fact that $(1 - \beta_k\mu_g)^2 < (1 - \beta_k\mu_G) < 1$, we have:
\begin{equation*}
\begin{split}
E[||v_{k} - v_{\lambda_{k}}||^2] \le& (1 + \gamma_v\alpha_{k-1})^3(1 - \beta_k\mu_G)E[||v_{k-1} - v_{\lambda_{k-1}}||^2] + \beta_k^2(1 + M^2)\sigma^2\\
&+ \left((1 + \frac{1}{\gamma_v\alpha_{k-1}})L_{F,\omega}^2 + (2 +
\gamma_v\alpha_{k-1} + \frac{1}{\gamma_v\alpha_{k-1}})M^2L_{G,\omega\omega}^2 \right)\beta_k^2(1 + \gamma_{\omega}\alpha_{k-1})E[||\omega_{k-1} - \omega_{\lambda_{k-1}}||^2]\\
&+ C_v^2(1 + \gamma_v\alpha_{k-1})^2(1 + \frac{1}{\gamma_v\alpha_{k-1}})\alpha_{k-1}^2E[||d_{k-1}||^2]\\
\end{split}
\end{equation*}
Finally suppose we have $\tau_k = c_{\tau}\alpha_{k-1}$, $\beta_k=c_{\beta}\alpha_{k-1}$ and $\eta_k=c_{\eta}\alpha_{k-1}$, and omit the higher order terms, we get:
\begin{equation*}
\begin{split}
E[||v_{k} - v_{\lambda_{k}}||^2] \le& (1 + \gamma_v\alpha_{k-1})^3(1 - \beta_k\mu_G)E[||v_{k-1} - v_{\lambda_{k-1}}||^2] + c_{\beta}^2\alpha_{k-1}^2(1 + M^2)\sigma^2\\
&+ \left(L_{F,\omega}^2 +M^2L_{G,\omega\omega}^2 \right)c_{\beta}^2\alpha_{k-1}/\gamma_vE[||\omega_{k-1} - \omega_{\lambda_{k-1}}||^2] + C_v^2\alpha_{k-1}/\gamma_vE[||d_{k-1}||^2]\\
\end{split}
\end{equation*}
\end{proof}

\begin{lemma}
With Assumption~\ref{upper_assumption},~\ref{lower_assumption} and~\ref{noise_assumption_append} hold, and $0<\tau_k < min(\frac{\mu_G}{L_{G,\omega}^2},\frac{1}{\mu_G})$, we have:
\begin{equation*}
\begin{split}
E[||\omega_{k} - \omega_{\lambda_{k}}||^2]
\le& (1 - \mu_G\tau_k)(1 + \gamma_\omega\alpha_{k-1})E[||\omega_{k-1} - \omega_{\lambda_{k-1}}||^2] +  C_{\omega, d}\alpha_{k-1}E[||d_{k-1}||^2] + C_{\omega, n}\alpha_{k-1}^2\sigma^2\\
\end{split}
\end{equation*}
Where $C_{\omega,d} = C_{\omega}^2/\gamma_\omega$, $C_{\omega,n} = c_{\tau}^2$, $\gamma_\omega$ and $c_{\tau}$ are some constants.
\label{lemma:16}
\end{lemma}
\begin{proof}
Firstly by the update rule of $\omega_k$, we have:
\begin{equation}
\label{eq:eq38}
\begin{split}
&E[||\omega_k - \omega_{\lambda_{k}}||^2] = E||\omega_{k} -  \tau_k\partial_{\omega}G(\lambda_k, \omega_k; \xi_{k,1}) - \omega_{\lambda_k}||^2]\\
=& E[||\omega_k - \omega_{\lambda_{k}}||^2] + \tau_k^2E||\partial_{\omega}G(\lambda_k, \omega_k; \xi_{k,1})||^2] -2\tau_k E[\langle \omega_k - \omega_{\lambda_{k}}, E[\partial_{\omega}G(\lambda_k, \omega_k; \xi_{k,1})]\rangle] \\
=& E[||\omega_k - \omega_{\lambda_{k}}||^2] + \tau_k^2E||\partial_{\omega}G(\lambda_k, \omega_k; \xi_{k,1})||^2] -2\tau_k E[\langle \omega_k - \omega_{\lambda_{k}}, \partial_{\omega}G(\lambda_k, \omega_k)\rangle] \\
\overset{(i)}{\le}& (1 - 2\mu_G\tau_k)E[||\omega_k - \omega_{\lambda_{k}}||^2] + \tau_k^2E[||\partial_{\omega}G(\lambda_k, \omega_k; \xi_{k,1})||^2] \\
\overset{(ii)}{\le}& (1 - 2\mu_G\tau_k)E[||\omega_k - \omega_{\lambda_{k}}||^2] + \tau_k^2E[||\partial_{\omega}G(\lambda_k, \omega_k) - \partial_{\omega}G(\lambda_k, \omega_{\lambda_k})||^2] + \tau_k^2\sigma^2\\
\overset{(iii)}{\le}& (1 - 2\mu_G\tau_k + \tau_k^2 L_{G,\omega}^2)E[||\omega_k - \omega_{\lambda_{k}}||^2] + \tau_k^2\sigma^2
\overset{(iv)}{\le} (1 - \mu_G\tau_k)E[||\omega_k - \omega_{\lambda_{k}}||^2] + \tau_k^2\sigma^2
\end{split}
\end{equation}
Where (i) uses the strong convexity of $G$ and the inequality for strongly convex function: $
\langle \omega_{k} - \omega_{\lambda_{k}}, \partial_{\omega}G(\lambda_k, \omega_k)\rangle \ge \mu_G||\omega_{k} - \omega_{\lambda_{k}}||^2$. $(ii)$ uses separating mean and variance and the definition of $\omega_{\lambda_k}$; $(iii)$ uses Assumption~\ref{lower_assumption}.b; (iv) uses the condition that $\tau_k < \mu_G/L_{G,\omega}^2$. Finally, we combine Eq.~(\ref{eq:eq38}) with Proposition~\ref{prop:6}, we get:
\begin{equation*}
\begin{split}
E[||\omega_{k} - \omega_{\lambda_{k}}||^2]
\le& (1 - \mu_G\tau_k)(1 + \gamma_\omega\alpha_{k-1})E[||\omega_{k-1} - \omega_{\lambda_{k-1}}||^2] +  \frac{C_{\omega}^2}{\gamma_\omega} \alpha_{k-1}E[||d_{k-1}||^2] + c_{\tau}^2\alpha_{k-1}^2\sigma^2\\
\end{split}
\end{equation*}
\end{proof}
\noindent Where we omit the term $1 + \gamma_{\omega}\alpha_{k-1}$ for $E[||d_{k-1}||^2]$, and use the fact that $\tau_k=c_{\tau}\alpha_{k-1}$. This completes the proof.

\subsection{Proof for the convergence Theorem 5}
\begin{theorem} (Theorem 5 in the main text)
With Assumption~\ref{upper_assumption},~\ref{lower_assumption} and~\ref{noise_assumption_append} hold, and suppose $\beta_k = c_{\beta}\alpha_{k-1}$, $\tau_k =c_{\tau}\alpha_{k-1}$, $\eta_k = c_{\eta}\alpha_{k-1}$, and $\alpha_k = \delta/\sqrt{k+1}$, then we have:
\begin{equation*}
\begin{split}
\frac{1}{K}\sum_{k=0}^{K-1}||\nabla f(\lambda_k)||^2
\le& \frac{2\Phi_{0}}{\delta\sqrt{K}} + \frac{2\delta\bar{C}\sigma^2\ln(K+1)}{\sqrt{K}}\\
\end{split}
\end{equation*}	
where $\gamma_v = 32\Gamma_1^2C_v^2$, $\gamma_{\omega} = 8\Gamma_2^2C_w^2$,  $c_{\tau} =(2\gamma_v\Gamma_2^2\gamma_{\omega} + 2\gamma_v\Gamma_2^2 + 4\Gamma_1^2c_{\beta}^2\left(L_{F,\omega}^2 + M^2L_{G,\omega\omega}^2 \right))/(\gamma_v\Gamma_2^2\mu_G)$,$c_{\eta} = 4\Gamma_2^2/L_f$, $c_{\beta}=2(3\gamma_v+1)/\mu_G$, $\bar{C} = C_{d,n}/c_{\eta} + 4\Gamma_1^2C_{v,n} + \Gamma_2^2C_{\omega,n}$, $\delta$ is a constant such that it satisfies Eq.~(\ref{eq:eq43}).
\end{theorem}
\begin{proof}
We firstly denote the following potential function:
\[\Phi_k = f({\lambda}_k) + \Gamma_2^2A_k + 4\Gamma_1^2B_k + C_k/c_{\eta}\]
where $A_k = ||\omega_{k} - \omega_{\lambda_k}||^2$, $B_k = ||v_{k} - v_{\lambda_{k}}||^2$, $C_k = ||\nabla f_k -  d_k||^2$. Next we bound $\Phi_{k+1} - \Phi_k$, which is composed of four terms, we get a bound for each of them based on Lemma~\ref{lemma:13}-\ref{lemma:16}.
By Lemma~\ref{lemma:13}, we have:
\begin{equation}
\label{eq:eq29}
\begin{split}
E[f(\lambda_{k+1}) - f(\lambda_k)] \le& -\frac{\alpha_k}{2}||\nabla f(\lambda_k)||^2 + \alpha_k\Gamma_2^2E[A_k] + 4\alpha_k\Gamma_1^2E[B_{k}] + \alpha_k E[C_k] - \frac{\alpha_k}{2}(1 - \alpha L_f) E[||d_k||^2]\\
\end{split}
\end{equation}
By Lemma~\ref{lemma:14}, we have:
\begin{equation}
\label{eq:eq30}
\begin{split}
&E[C_{k+1} - C_{k}] \le -c_{\eta}\alpha_{k}E[C_{k}] + C_{d,d}\alpha_{k}^2E[||d_{k}||^2] + C_{d,\omega}\alpha_{k}^2E[A_k] + C_{d,v}\alpha_{k}^2E[B_{k}] + C_{d,n}\alpha_{k}^2\sigma^2
\end{split}
\end{equation}
By Lemma~\ref{lemma:15}, we have:
\begin{equation}
\label{eq:eq31}
\begin{split}
E[B_{k+1} - B_{k}] \le& ((1 + \gamma_v\alpha_{k})^3(1 - \beta_{k+1}\mu_G) - 1)E[B_{k}] + C_{v,\omega}\alpha_{k}E[A_k] + C_{v,d}\alpha_{k}E[||d_{k}||^2] + C_{v,n}\alpha_{k}^2\sigma^2\\
\end{split}
\end{equation}
By Lemma~\ref{lemma:16}, we have:
\begin{equation}
\label{eq:eq32}
\begin{split}
&E[A_{k+1}- A_k] \le ((1 - \mu_G\tau_{k+1})(1 + \gamma_\omega\alpha_{k}) - 1)E[A_k] + C_{\omega,d} \alpha_{k}E[||d_{k}||^2] + C_{\omega,n}\alpha_{k}^2\sigma^2
\end{split}
\end{equation}
Sum inequalities Eq.~(\ref{eq:eq29})-~(\ref{eq:eq32}), we have:
\begin{equation}
\label{eq:eq39}
\begin{split}
\Phi_{k+1} - \Phi_k
\le& -\frac{\alpha_k}{2}||\nabla f(\lambda_k)||^2 + \left(C_{d,n}/c_{\eta} + 4\Gamma_1^2C_{v,n} + \Gamma_2^2C_{\omega,n}\right)\alpha_k^2\sigma^2
\end{split}
\end{equation}
The terms related to $E[||d_k||^2]$, $E[A_k]$ and $E[B_{k}]$ are eliminated by using the facts shown in Eq.~(\ref{eq:eq40})-~(\ref{eq:eq42}). Firstly, for terms related to $E[||d_k||^2]$, we have:
\begin{equation}
\label{eq:eq40}
\begin{split}
&\frac{2\Gamma_2^2}{c_{\eta}}\alpha_k^2 + \frac{4\Gamma_1^2C_v^2}{\gamma_v}\alpha_k + \frac{\Gamma_2^2C_{\omega}^2}{\gamma_\omega}\alpha_k < \frac{\alpha_k}{2}(1 - \alpha_k L_f)
\end{split}
\end{equation}
Next for the term related to $E[B_{k}]$, we have:
\begin{equation}
\label{eq:eq41}
\begin{split}
&4\Gamma_1^2((1 + \gamma_v\alpha_k)^3(1 - \beta_{k+1}\mu_G) -1) + 32c_{\beta}^2\Gamma_1^2\Gamma_3^2\alpha_k^2/c_{\eta} + 4\Gamma_1^2\alpha_k < 0 
\end{split}
\end{equation}
Finally, for the term involved with $E[A_k]$, we have:
\begin{equation}
\label{eq:eq42}
\begin{split}
&\Gamma_2^2((1 + \gamma_\omega\alpha_k)(1 - \mu_G \tau_{k+1}) - 1) + 2(c_{\tau}^2\Gamma_2^2\Gamma_3^2 + 4c_{\beta}^2\Gamma_1^2\Gamma_4^2)\alpha_k^2/c_{\eta} + 4\Gamma_1^2c_{\beta}^2\left(L_{F,\omega}^2 + M^2L_{G,\omega\omega}^2 \right)\alpha_k/\gamma_v + \Gamma_2^2\alpha_k < 0
\end{split}
\end{equation}
We can verify the correctness of the inequalities Eq.~(\ref{eq:eq40})-~(\ref{eq:eq42}) by taking the corresponding values of $\gamma_v = 32\Gamma_1^2C_v^2$, $\gamma_{\omega} = 8\Gamma_2^2C_w^2$, $c_{\eta} = 4\Gamma_2^2/L_f$, $c_{\beta}=2(3\gamma_v+1)/\mu_G$, $c_{\tau} =(2\gamma_v\Gamma_2^2\gamma_{\omega} + 2\gamma_v\Gamma_2^2 + 4\Gamma_1^2c_{\beta}^2\left(L_{F,\omega}^2 + M^2L_{G,\omega\omega}^2 \right))/(\gamma_v\Gamma_2^2\mu_G)$ as stated in the Theorem and let $\delta$ satisfy the following condition:
\begin{equation}
\label{eq:eq43}
\begin{split}
\delta < min\bigg(\frac{1}{4L_f}, \frac{2\Gamma_2^2\mu_G c_{\beta}}{12\Gamma_2^2\gamma_v^2 + 8\Gamma_3^2L_fc_{\beta}^2}, \frac{c_{\eta}\Gamma_2^2(1 + \gamma_{\omega})}{2(c_{\tau}^2\Gamma_2^2\Gamma_3^2 + 4c_{\beta}^2\Gamma_1^2\Gamma_4^2) + \Gamma_2^2c_{\eta}}, \frac{1}{c_{\eta}}, \frac{1}{2(3\gamma_v+1)}, \frac{1}{c_{\tau}\mu_G}, \frac{\mu_G}{c_{\tau}L_{G,\omega}^2}\bigg)
\end{split}
\end{equation}
Now we summarize both sides of Eq.~(\ref{eq:eq39}), and rearrange the terms:
\begin{equation*}
\begin{split}
\sum_{k=0}^{K-1} \frac{\alpha_k}{2}||\nabla f(\lambda_k)||^2 \le& (\Phi_{0} - \Phi_K) + \bar{C}\sum_{k=0}^{K-1}\alpha_k^2\sigma^2\\
\end{split}
\end{equation*}
We denote $\bar{C} = C_{d,n}/c_{\eta} + 4\Gamma_1^2C_{v,n} + \Gamma_2^2C_{\omega,n}$. Suppose we take $\alpha_k = \frac{\delta}{\sqrt{k+1}}$ where $\delta$ is chosen \emph{s.t.} Eq.~(\ref{eq:eq43}) is satisfied, we get:
\begin{equation*}
\begin{split}
\frac{\delta}{\sqrt{K}}\sum_{i=0}^{K-1}||\nabla f(\lambda_k)||^2 \le \sum_{i=0}^{K-1}\frac{\delta}{\sqrt{k}}||\nabla f(\lambda_k)||^2
\le& 2\Phi_{0} + 2\bar{C}\delta^2\sigma^2\ln(K+1)\\
\end{split}
\end{equation*}
Divide both sides by $K$, we have:
\begin{equation*}
\begin{split}
\frac{1}{K}\sum_{i=0}^{K-1}||\nabla f(\lambda_k)||^2
\le& \frac{2\Phi_{0}}{\delta\sqrt{K}} + \frac{2\delta\bar{C}\sigma^2\ln(K+1)}{\sqrt{K}}\\
\end{split}
\end{equation*}	
\end{proof}

\section{More Details about the Hyper Data-cleaning Experiment}
\label{Appendix:E}
In the Hyper Data-cleaning task, we test over MNIST~\cite{lecun2010mnist}, Fashion-MNIST~\cite{xiao2017fashion} and QMNIST~\cite{yadav2019cold}. The datasets are created as described in the main text. 
In experiments, we use batch-size 256 and run 2000 hyper-iterations. Other hyper-parameters are chosen as follows: we pick $\delta = 1000$, $c_{\tau} = 0.001$, $c_{\beta} = 10^{-4}$ and $c_{\eta} = 9 * 10^{-4}$ for our FSLA. For BP method, we pick $\delta = 1000$, $c_{\tau} = 0.001$ and $c_{\eta} = 9 * 10^{-4}$; For NS method, we choose $\delta = 1000$, $c_{\tau} = 0.001$, $c_{\eta} = 9 * 10^{-4}$ and $\beta$ as 0.1; For CG method, we choose $\delta = 1000$, $c_{\tau} = 0.001$, and $c_{\eta} = 9 * 10^{-4}$. Next, we present some ablation studies of the hyper-parameters for FSLA, and the results are shown in Figure~\ref{fig:aba1} and Figure~\ref{fig:aba2}. In each plot, we vary one hyper-parameter and fix other hyper-parameters whose values are taken as just mentioned. We do not include hyper-parameters that diverge. For example, when the batch-size is 32 or when $c_{\beta}$ is $5\times10^{-4}$, the algorithm does not converge.

\begin{figure}[ht]
\begin{center}
\includegraphics[width=0.33\columnwidth]{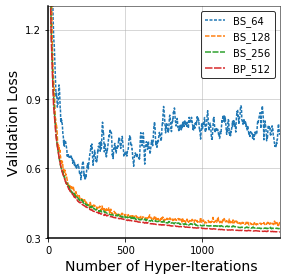}
\includegraphics[width=0.33\columnwidth]{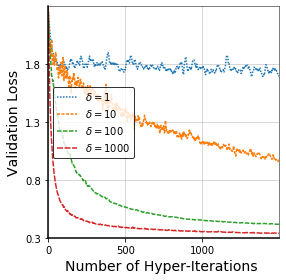}
\end{center}
\caption{Ablation Study for Batch Size (Left) and outer learning rate coefficient $\delta$ (Right).}
\label{fig:aba1}
\end{figure}

\begin{figure}[ht]
\begin{center}
\includegraphics[width=0.33\columnwidth]{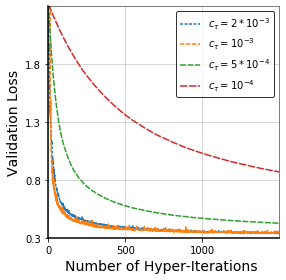}
\includegraphics[width=0.33\columnwidth]{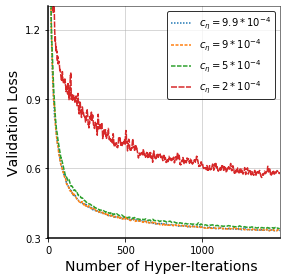}
\includegraphics[width=0.33\columnwidth]{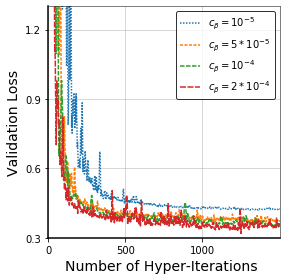}
\end{center}
\caption{Ablation Study for $c_{\tau}$ (Left) and outer $c_{\eta}$ (Middle)a and  $c_{\beta}$ (Right).}
\label{fig:aba2}
\end{figure}

\end{document}